\documentclass{article}

\PassOptionsToPackage{numbers}{natbib}

\usepackage[preprint]{neurips_2026}

\usepackage[utf8]{inputenc} %
\usepackage[T1]{fontenc}    %
\usepackage{hyperref}       %
\usepackage{url}            %
\usepackage{booktabs}       %
\usepackage{amsfonts}       %
\usepackage{nicefrac}       %
\usepackage{microtype}      %
\usepackage{xcolor}         %
\usepackage{amsmath,amsthm,amsfonts}
\usepackage{algorithm}
\usepackage{algpseudocode}
\usepackage{amssymb}
\usepackage{multirow}
\usepackage{xspace}
\usepackage{arydshln}
\usepackage{graphicx}
\usepackage{cleveref}
\usepackage{wrapfig}
\usepackage{placeins}
\usepackage{subcaption}

\newtheorem{lemma}{Lemma}
\newtheorem{proposition}{Proposition}

\newcommand{\ours}{GASLoC\xspace}
\newcommand{\oursbold}{\textbf{\ours}\xspace}
\newcommand{\ddp}{DDP\xspace}
\newcommand{\localsgd}{Local SGD\xspace}
\newcommand{\diloco}{DiLoCo\xspace}
\newcommand{\DAdam}{DAdam\xspace}
\newcommand{\allreduce}{All-Reduce\xspace}
\newcommand{\versionD}{Local-\DAdam}

\title{Unifying Local Communications and Local Updates for LLM Pretraining}

\author{
\parbox{\textwidth}{\centering
  Pietro Cagnasso$~^{1,2}$
  \hspace{5pt}\qquad
  Eugene Belilovsky$^*~^{1,2}$
  \hspace{5pt}\qquad
  Edouard Oyallon$^*~^{3}$
  \vspace{5pt}\\
  {\normalfont
  $^1$Concordia University ~~ $^2$Mila ~~~~ $^3$CNRS, Sorbonne University}\\
}
}

\begin{document}

\maketitle
\begingroup
\renewcommand{\thefootnote}{\fnsymbol{footnote}}
\makeatletter
\long\def\@makefntext#1{%
  \parindent=0pt
  \noindent
  \makebox[0.8em][l]{\@makefnmark}#1%
}
\makeatother
\footnotetext[1]{Equal contributions.}
\endgroup

\begin{abstract}

  Communication-efficient pre-training of LLMs is increasingly important as training draws on compute distributed across clusters, data centers, and lower-bandwidth links. Many practical methods reduce communication frequency but still rely on synchronous \allreduce operations that maintain identical model states and tie progress to global collectives. 
  This can become a bottleneck when bandwidth or worker speed is heterogeneous. We introduce \oursbold, a novel decentralized pre-training algorithm that generalizes the notion of communication acceleration to the recently popular ``outer optimizer'' to allow a practical gossip-based training framework that is compatible with adaptive optimizers, allows for local optimizer steps, and can utilize sparse randomized peer communication. 
  Empirically, 
  on a number of standard LLM training tasks, we demonstrate that \oursbold  outperforms state-of-the-art decentralized algorithms in single step per communication setting for a number of topologies and, unlike existing decentralized methods in the LLM setting, iit achieves performance competitive with DiLoCo when utilizing multiple local steps. In the heterogeneous bandwidth setting, we demonstrate the %
  advantage of \oursbold showing that it can significantly outperform DiLoCo.

\end{abstract}

\section{Introduction}
Communication is a major bottleneck in distributed optimization and can significantly limit the scalability of distributed training, in particular when a large number of compute nodes are available and for training very large models like LLMs. In typical large-scale training setups, communication is implemented through bandwidth-efficient \allreduce, which typically allows all nodes to exchange information and average in parallel through two phases: a reduce-scatter phase followed by an all-gather (broadcast) phase. Although this strategy is highly effective on tightly coupled clusters, its communication cost still scales linearly with the number of participating nodes and is highly synchronous. In practice, this can saturate the bandwidth and lead to a lack of robustness to slower workers. Recent practical work on communication-efficient training~\citep{douillard2023diloco} has shown a number of methods that utilize less frequent communication~\citep{charles2026communication}, gradient compression~\citep{wang2023cocktailsgd,sarfi2025sparseloco}, and communication overlap~\citep{douillard2025streaming}. However, all these methods rely on \allreduce operations.

Decentralized learning based on gossip algorithms  is a direction of research widely studied in classical optimization that replaces global synchronization with communication over  network topologies. This makes it attractive for large-scale or bandwidth-constrained settings, since each worker communicates only with a small subset of neighbors. In principle, because each step involves fewer communication links than \allreduce, a single gossip step can substantially reduce communication bottlenecks relative to standard \allreduce, whose speed is limited by the slowest worker. Furthermore, maintaining optimization performance often requires more communication rounds between model updates. This creates a trade-off similar to that of \localsgd~\citep{stich2018local} and \diloco-style~\citep{douillard2023diloco} methods: communication becomes cheaper per round, but may need to occur more frequently.

A central quantity in gossip-based optimization is the consensus error, which measures disagreement between local model replicas as workers alternate between local updates and  communication. Classical communication-acceleration methods reduce this disagreement by applying momentum to the consensus dynamics~\citep{loizou2019provably,berthier2020accelerated}. The combination of gossip and local updates has been studied before: \citet{koloskova2020unified} provide a unified theory covering decentralized SGD with local updates, time-varying topologies, and heterogeneous local objectives. However, this line of work primarily treats local computation through a shared update schedule and studies heterogeneity as statistical heterogeneity across workers. In the specific setting of decentralized LLM training with modern adaptive optimizers, the closest prior work we are aware of is \citet{wang2024promise}. Their method communicates after each local optimizer step and does not use an outer momentum mechanism for communication acceleration or a \localsgd-style block of multiple local updates between communication rounds. In contrast, our setting is homogeneous-data LLM pre-training with system heterogeneity, where workers may have different bandwidth or compute speeds and the number of local steps \(H_i\) can be adapted across each worker $i$. We therefore study whether a decentralized \localsgd-style method with sparse randomized peer communication \citep{wang2024promise,ying2021exponential} and an outer momentum mechanism can remain stable in modern LLM training while reducing dependence on global synchronization.

In this work, we propose a novel, principled algorithm that combines both local gradient steps and local communication rounds alongside communication acceleration. In contrast to standard decentralized methods, our algorithm communicates locally updated parameters, and decouples communication from model updates through a momentum mechanism.
\ours improves over evaluated decentralized baselines in the \(H=1\) regime, and remains close to \diloco in the \(H=30\) regime while avoiding global collectives, and yields substantial wall-clock gains under a bandwidth-straggler model.
It can be seen as a principled generalization of FedOpt frameworks~\citep{reddi2021adaptive}. Moreover, the outer optimizer takes the role of the classical and well-understood communication acceleration, particularly when the graph is not complete. It reduces to a standard outer/server optimizer when, instead, the graph is complete,
elegantly linking this important and well-understood structure from gossip methods to the emerging practical results in \citet{reddi2021adaptive} and  \citet{douillard2023diloco}.

\paragraph{Contributions.}
Our contributions in this work are as follows. First, \textbf{(a)} we propose a novel decentralized algorithm that practically incorporates communication acceleration  and \textbf{(b)} we show that it
is a strict generalization of DiLoCo to gossip-based algorithms. \textbf{(c)} 
We argue that, instead of averaging all available neighbors, 
using a subset of \(k \leq 2\) neighbors is more communication- and algorithmically-efficient than standard All-Reduce.
\textbf{(d)} We prove standard convergence rates in the homogeneous setting.
\textbf{(e)}  We show for the first time in the literature that in gossip settings the outer momentum mechanism reduces the communication complexity, improving from $\chi$ to $\sqrt{\chi}$, where $\chi$ denotes the spectral gap. We empirically demonstrate \textbf{(f)} that we
outperform state-of-the-art decentralized methods for LLMs, while \textbf{(g)}
also extending to the multi-step setting.
In this setting, \ours yields competitive results compared to \diloco, while offering communication advantages, robustness to heterogeneous bandwidth, and better fault tolerance.

\section{Related Work}

\paragraph{\localsgd and extensions to LLM training.}
\localsgd~\citep{stich2018local} is a standard approach in distributed training to reduce communication frequency. It has been shown to work well for LLM training when combined with adaptive optimizers and an outer momentum mechanism, as in \diloco~\citep{douillard2023diloco}. In both settings, the main idea is to maintain an outer loop that aggregates parameters through communication, while an inner loop performs local gradient updates. Although these ideas originally arose in the context of federated learning~\citep{konevcny2016federated}, several extensions have been proposed since then, including methods that incorporate a slower momentum mechanism~\citep{wang2019slowmo}. Another line of work seeks to further reduce communication through sparse variants, which compress transmitted updates to lower communication costs~\citep{douillard2025streaming,kale2025eager}. 
A related recent work \citep{kolehmainen2025noloco} considers extensions of \diloco that allow for randomization of the pipeline path and for averaging pipeline stages in subgroups,
which is similar in spirit to decentralized methods.

\begin{figure}[t]
    \centering
    \includegraphics[width=0.9\linewidth]{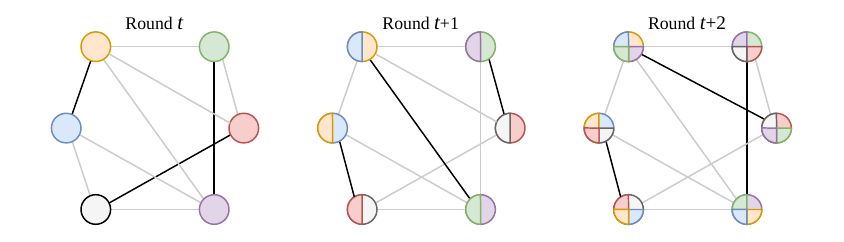}
    \caption{\textbf{Time-varying 1-Peer gossip communication.} At each round, only a sparse subset of peer-to-peer exchanges is active, shown in black, while the possible communication graph is shown in light gray. Changing the active peers across rounds lets information propagate through the network without global synchronization. Here, each worker communicates with one peer per round. When \ours communicates on this kind of graph we name it \ours-1-Peer.}
    \label{fig:gossip}
\end{figure}

\paragraph{Decentralized (gossip) learning.}
Decentralized learning~\citep{koloskova2020unified} aims to remove the need for a central coordinator by allowing communication only between neighboring nodes.
Much of this literature is motivated by settings with data heterogeneity, where different workers or clients optimize different local objectives due to non-IID data partitions, as in federated learning ~\citep{mcmahan2017communication} and decentralized training on heterogeneous datasets ~\citep{vogels2021relaysum}. Existing theory also distinguishes this regime from the homogeneous-data case~\citep{koloskova2020unified}.
Our work instead targets LLM pre-training, where workers sample from a shared large-scale corpus and data are typically treated as homogeneous across workers. Thus, the main heterogeneity we study is not statistical heterogeneity across local objectives, but system heterogeneity: uneven bandwidth, variable compute speed, and stragglers. This paradigm has been extensively studied in the theoretical setting. However, it has seen limited practical adoption, in particular with LLM training. 
In particular, we show that a straightforward decentralized adaptation of \diloco can be highly sensitive to the number of local steps. \citet{chen2021accelerating} also consider gossip-based communication, but their method does not incorporate a momentum step at communication time, whereas our approach does, which empirically leads to more stable training. Similarly, \citet{assran2019stochastic} communicate parameters at every gradient step rather than relying on an outer-loop structure. Motivated by these observations, we introduce an alternative algorithm that is significantly more stable in this regime. We also note that \citet{wang2024promise} and \citet{nabli2023textbf} are among the first works to demonstrate the promise of decentralized adaptive methods in a cluster setting, but they are not able to incorporate local gradient steps.
Randomized and asynchronous decentralized algorithms can further reduce or reshape communication by sampling edges, using stochastic clocks, or optimizing averaging weights~\citep{boyd2006randomized,even2021continuized,nabli2023dadao,nabli2025decentralized}. Although our \(2\)-Peer rule also samples communication partners, we use this randomization as a simple sparse mixing rule within a \localsgd-style outer loop, leaving optimized sampling distributions \citep{boyd2006randomized,nabli2025decentralized}, activation rates, and asynchronous protocols outside the scope of this work.\vspace{-4pt}

\paragraph{Alternatives to \allreduce.}
Several approaches aim to reduce the communication overhead of \allreduce-based training. One line of work explores sparse or structured communication topologies, including expander-like strategies, to improve communication efficiency~\citep{vogels2021relaysum}. Another direction is elastic synchronization: Elastic Averaging SGD (EASGD) introduces a central reference variable together with an elastic coupling between local models, which can improve robustness and reduce synchronization pressure~\citep{zhang2015deep}. More recent work also revisits elastic communication mechanisms for stable large-scale training~\citep{kang2025elaswave}, as well as TorchFT at the implementation level. Our approach can be viewed as a form of randomized gossip~\citep{boyd2006randomized}. However, instead of sampling edges independently, we sample a random permutation to define the communication pattern, which requires agreement between workers. Related low-degree communication schemes have been studied before, including $1$-Peer exponential graphs~\citep{ying2021exponential} and Alternating Exponential Rings~\citep{wang2024promise}. These methods are deterministic and are typically designed to match a prescribed or predictable network topology (e.g., computing clusters). As a result, they may suffer from less favorable  mixing guarantees. In contrast, our method relies on randomization, which improves the effective spectral gap of the communication operator.

\section{Method}

\subsection{Generalized Accelerated Sparse Communication Local Computation (\ours) }
\paragraph{Notation.}
We consider $n$ workers with local models \(x_i \in \mathbb{R}^d\). We write \(x = (x_1,\ldots,x_n)\) for the collection of local models and denote by \(\bar{x}\) their average. Let \(\pi \triangleq I-\frac1n\mathbf 1\mathbf 1^\top\) be the projector onto the disagreement subspace. We encode communication by a weighted graph Laplacian
\[
    \Lambda \triangleq \frac 12\sum_{(i,j)\in\mathcal E}\lambda_{ij}(e_i-e_j)(e_i-e_j)^\top, \qquad \text{ with }\sum_{i,(ij)\in\mathcal{E}}\lambda_{ij}=1 \text{ and }\lambda_{ij}>0\,,
\]
 so that \(\Lambda \mathbf 1=0\) and also let $\mathcal{N}_i=\{j,(ij)\in\mathcal E\} \cup \{i\}$. In this case, the graph connectivity is determined by \(\chi\geq 1\), denoting the inverse of the smallest nonzero eigenvalue, the spectral gap~\citep{nabli2023dadao}.

\paragraph{\ours.}
We describe \ours in the inner-loop setting for SGD, while \Cref{alg:ours-general} gives a more general formulation for arbitrary inner/outer optimizers. In parallel, each node performs local gradient steps, and an outer loop periodically aggregates the resulting updates. The number of local steps \(H_i\) may be adapted across nodes to compensate for a straggler worker. More precisely, for each node \(i\), the SGD version of the method is written as follows, where \(\eta,\alpha,\beta>0\) are learning rates:
\begin{equation}
    \begin{array}{c@{\qquad}c}
        \textbf{Inner loop} & \textbf{Outer loop}
        \\[0.4em]
        \left\{
        \begin{aligned}
            x_{t,0}^i &= x_t^i,\\
            x_{t,h+1}^i &= x_{t,h}^i - \beta \nabla f(x_{t,h}^i),
            \qquad h=0,\ldots,H_i-1,
        \end{aligned}
        \right.
        &
        \left\{
        \begin{aligned}
            g_t^i &\triangleq x^i_{t,H_i}-x^i_{t,0},\\
            x_{t+1} &= x_t + \eta g_t - \alpha \Lambda\!\left(x_t+\eta g_t\right).
        \end{aligned}
        \right.
    \end{array}
\end{equation}

When \(\Lambda=\pi\), this algorithm recovers \diloco exactly, and therefore can be viewed as a generalization of it. Compared with \DAdam \citep{wang2024promise}, our method also allows the use of outer momentum, potentially enabling \textit{communication acceleration}, together with local SGD steps as discussed below. Our method can also be combined with momentum, so that the outer update then becomes
\begin{algorithm}[t]
\caption{\ours}
\label{alg:ours-general}
{\footnotesize
\begin{algorithmic}[1]
\State \textbf{Input:} inner steps $\{H_i\}_{i=1}^n$, inner optimizer $\mathrm{Optimizer}^{\mathrm{in}}$, outer optimizer $\mathrm{Optimizer}^{\mathrm{out}}$
\For{$t=0,1,\dots,T-1$}
    \For{each worker $i\in\{1,\dots,n\}$ \textbf{in parallel}}
        \State $x_i^{(t,0)} \gets x_i^{(t)}$
        \For{$h=0,1,\dots,H_i-1$}
            \State $x_i^{(t,h+1)} \gets 
            \mathrm{Optimizer}^{\mathrm{in}}\!\left(
            x_i^{(t,h)}, \nabla f(x_i^{(t,h)})
            \right)$
        \EndFor
        \State $y_i^{(t)} \gets x_i^{(t,H_i)}$
    \EndFor

    \For{each worker $i\in\{1,\dots,n\}$ \textbf{in parallel}}
        \State Sample neighbors $\mathcal{N}_i^{(t)}$, always including node $i$, and receive $\{y_j^{(t)}:j\in \mathcal{N}_i^{(t)}\}$
        
        \State $x_i^{(t+1)} \gets 
        \mathrm{Optimizer}^{\mathrm{out}}\!\left(x_i^{(t)}, \frac{1}{|\mathcal{N}_i^{(t)}|}
        \sum_{j\in\mathcal{N}_i^{(t)}} \left(x_i^{(t)} - y_j^{(t)}\right)\right)$
    \EndFor
\EndFor
\State \textbf{Output:} $\displaystyle \bar{x}^{(T)}=\frac{1}{n}\sum_{i=1}^n x_i^{(T)}$
\end{algorithmic}
}
\end{algorithm}
\begin{equation}
x_{t+1}
=
x_t
+
\eta g_t
-
\alpha \Lambda\!\left(x_t+\eta g_t\right)
+
\gamma
\Big[
\left(x_t+\eta g_t\right)
-
\left(x_{t-1}+\eta g_{t-1}\right)
\Big],
\qquad t \ge 1,
\label{eq:momentum-update}
\end{equation}
where \(\gamma>0\) is the momentum parameter. Here, momentum is applied to the post-local-update iterates \(x_t+\eta g_t\), which can accelerate communication while preserving the same decentralized mixing structure. However, despite being faster, this approach still has at least two drawbacks, for instance in the case of the complete graph: it is not robust to faulty communication, and  it requires significantly more communication. This motivates the notion of randomized peer-to-peer communication.

\subsection{Randomized peer-to-peer communication}
\label{sec:k-peer}
Classical $1$-Peer communication schemes typically rely on carefully constructed graph sequences tailored to the underlying network topology \cite{ying2021exponential,wang2024promise}. While these schemes are attractive due to their low per-round communication cost, their mixing properties can be poor: exponential graphs yield a spectral gap scaling as $\log n$, whereas cycles  lead to a spectral gap scaling as $n^2$, thereby slowing convergence. Such sparse communication patterns remain highly desirable in heterogeneous settings, where limiting communication and synchronization is essential. Our approach preserves this low communication cost while leveraging randomization \cite{nabli2023textbf} to improve the average mixing behavior.
\paragraph{Randomized $1$-Peer and $2$-Peer communication schemes.}
We propose randomized peer communication to reduce synchronization costs while preserving good mixing properties. For simplicity, assume that $n=2m$. At each communication round, we sample a permutation $\sigma:\{1,\dots,n\}\rightarrow\{1,\dots,n\}$ uniformly at random. The Laplacian of the corresponding \emph{$1$-Peer communication graph} is
\[
    \Lambda^{(1)}_\sigma
    =
    \frac12
    \sum_{i=1}^{m}
    \bigl(e_{\sigma(2i-1)} - e_{\sigma(2i)}\bigr)
    \bigl(e_{\sigma(2i-1)} - e_{\sigma(2i)}\bigr)^\top .
\]
\Cref{fig:gossip} illustrates an example of the randomized 1-Peer case, where each round activates
a different sparse matching.

We also consider a \emph{$2$-Peer communication graph}, obtained by connecting the nodes in the cyclic order induced by $\sigma$. Namely, using the convention $\sigma(n+1)=\sigma(1)$, define
\[
    \Lambda^{(2)}_\sigma
    =
    \frac12
    \sum_{i=1}^{n}\frac 12
    \bigl(e_{\sigma(i)} - e_{\sigma(i+1)}\bigr)
    \bigl(e_{\sigma(i)} - e_{\sigma(i+1)}\bigr)^\top .
\]

\begin{proposition}
Let $\sigma$ be sampled uniformly at random from the  permutations of $\{1,\dots,n\}$. Then
\[
    \mathbb{E}_\sigma\bigl[\Lambda^{(1)}_\sigma\bigr]
    =\mathbb{E}_\sigma\bigl[\Lambda^{(2)}_\sigma\bigr]
    =
    \frac{n}{2(n-1)} \pi,
    \quad
    \mathbb{E}_\sigma\bigl[ \bigl(\Lambda^{(1)}_\sigma\bigr)^2\bigr]=\frac{n}{2(n-1)}\pi,  \quad  \mathbb{E}_\sigma\bigl[ \bigl(\Lambda^{(2)}_\sigma\bigr)^2\bigr]=\frac{3n}{8(n-1)}\pi 
\]
\end{proposition}

Consequently, in expectation, the spectral gap of both the $1$-Peer and $2$-Peer communication schemes scales as $1$. This contrasts with a fixed cycle graph, for which $\chi$ scales as $n^2$ and becomes inefficient for large $n$.  Moreover, the lower variance of the $2$-Peer scheme explains its  gains over $1$-Peer.

\begin{figure}[t]
\label{fig:straggler}
    \centering
    \includegraphics[width=0.9\linewidth]{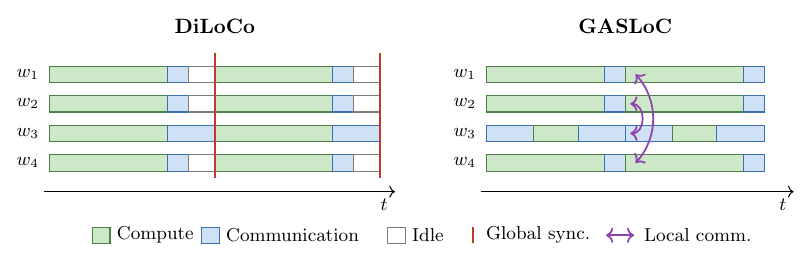}
    \caption{\textbf{Bandwidth-straggler scheduling.} Left: in \diloco implemented with an \allreduce, all workers perform the same number of local steps and the faster workers remain idle while waiting for the bandwidth-limited worker \(w_3\) at the global synchronization barrier. Right: \ours uses sparse peer exchanges and allows the bandwidth-limited worker to use fewer local steps \(H_3<30\), reducing its cycle time without forcing all workers to wait at a global barrier.
    }
    \label{fig:adapting-h}
\end{figure}

\paragraph{Robustness to heterogeneous bandwidth in randomized peer-to-peer communication settings.}
Another distinctive feature of our algorithm is that each worker can perform a flexible number of local steps with equal batch size, denoted by \(H_i\). This makes the method naturally robust to bandwidth heterogeneity, including slower compute nodes and workers with weaker or more variable network connections. Assume that worker \(i\) has communication time \(T_i^{\mathrm{comm}}\), defined as an upper bound on the time required to exchange the intermediate variable with another worker, and per-batch computation time \(T_i^{\mathrm{comp}}\).
This timing model treats each activated peer exchange as bidirectional communication with a symmetric per-worker upper bound. It therefore does not model directed asymmetric protocols, such as stochastic-push variants. This  matches our algorithm, whose communication graph is undirected at each round, and is sufficient for isolating the effect of sparse peer exchanges and worker-specific local step counts.
In an \allreduce-based method, one synchronization round can be bounded as
\[
T_{\mathrm{step}}^{\mathbf{AllReduce}}
=
\max_i \left\{ 2 T_i^{\mathrm{comm}} \right\}
+
\max_i \left\{ H_i T_i^{\mathrm{comp}} \right\}.
\]
The first term corresponds to the reduction and broadcast phases, while the second term reflects the need to wait for the slowest local computation. By contrast, in our decentralized method, each worker communicates only with 1 or 2 randomly selected neighbors at each synchronization round. The corresponding round duration is bounded by
\[
T_{\mathrm{step}}^{\oursbold,k}
=
\max_i
\left\{
kT_i^{\mathrm{comm}} + H_i T_i^{\mathrm{comp}}
\right\}.
\]
Therefore, since $k\leq 2$, we get
\[
T_{\mathrm{step}}^{\mathbf{\ours},k}
\leq
T_{\mathrm{step}}^{\mathbf{AllReduce}}.
\]
Indeed, \allreduce may be bottlenecked by different workers for communication and computation, whereas our method depends only on the largest combined per-worker cost.
This comparison still allows collective implementations: if the sparse schedule groups workers into disjoint islands, each island can use a small local \allreduce instead of participating in a global collective.
This advantage is particularly useful under heterogeneous hardware. For a prescribed maximum number of local steps \(H=\max_i H_i\), hardware utilization is improved by choosing the local step counts \(H_i\) so that workers complete a synchronization round in approximately the same time:
\[
T_1^{\mathrm{comm}} + H_1 T_1^{\mathrm{comp}}
=
\cdots
=
T_n^{\mathrm{comm}} + H_n T_n^{\mathrm{comp}}.
\]
Thus, faster or better-connected workers can perform more local updates, while slower or poorly connected workers perform fewer. This reduces idle time at synchronization points and improves robustness to heterogeneous compute and communication resources. \Cref{fig:adapting-h} illustrates this mechanism: with \allreduce, fast workers wait at each global barrier, whereas \ours lowers the bandwidth-limited worker's \(H_i\) and uses sparse peer exchanges.%

\paragraph{Fault tolerance.}
Sparse communication can improve robustness to stragglers and transient failures: each synchronization round depends only on the activated peer exchanges, so a slowdown affects only nearby workers, assuming failed exchanges can be skipped or rescheduled. This is complementary to Decoupled \diloco~\citep{douillard2026decoupled}, which uses independent learners and a central synchronizer with quorum-based updates. Here, there is no global synchronizer; resilience comes from replacing all-worker collectives with sparse peer-to-peer dependencies. In contrast, \allreduce requires every worker to join each collective step, making it more sensitive to stragglers and temporary failures.

\subsection{Convergence of \ours}
Appendix~\ref{app:math:convergence} proves the following convergence result for different communication schemes.

\begin{proposition}\label{prop:cvg}
Let \(x^\star\in\arg\min f\), and suppose that \(x^\star\) is an unconstrained
minimizer, so that \(\nabla f(x^\star)=0\). Suppose that, for every \(\xi\),
\(F(\cdot;\xi)\) is \(L\)-smooth. Assume moreover that 
\[
\mathbb{E}_\xi[\nabla F(x;\xi)]=\nabla f(x),
\qquad
\mathbb{E}_\xi\bigl[
\|\nabla F(x;\xi)-\nabla f(x)\|^2
\bigr]\le \sigma^2 .
\]
Assume
\[
0<\alpha<1,
\qquad
0<\beta\le \frac1{8L},
\qquad
0<\eta\le \frac{\alpha}{3\sqrt{3}\beta\rho L}.
\]
The parameter \((\rho,\gamma)\) is set to \((\chi,0)\) for standard gossip, to \((\sqrt{\chi},(1-\sqrt{\alpha/\chi})^2)\) for accelerated gossip, to \(\bigl(\frac{n}{n-1},0\bigr)\) for  \(1\)-Peer, and to
$
\left(
\frac{\alpha}{
1-\sqrt{
1-\frac{n}{n-1}\alpha+\frac{3n}{8(n-1)}\alpha^2
}},
0
\right)
$ for  \(2\)-Peer.

If the initialization satisfies $x_0^1=\cdots=x_0^n=x_0,$ then, for every \(T\ge 1\),
\[
\frac1T
\sum_{t=0}^{T-1}
\mathbb{E}\bigl[\|\nabla f(\bar x_t)\|^2\bigr]
\le
\frac{4(f(x_0)-f(x^\star))}{T\eta\beta}
+
\frac{2L\eta\beta\sigma^2}{n^2}
\sum_{i=1}^n H_i
+
60\beta^2L^2\sigma^2 .
\]
\end{proposition}
This proposition shows that, by appropriately tuning the step sizes \(\beta,\eta\), the residual terms in the convergence bound can be made arbitrarily small, which is the expected behavior for this type of method. The proof relies on a Lyapunov function which combines an optimization term, \(f(\bar x)-f(x^\star)\), that measures progress of the network average, and a consensus term that controls the disagreement between nodes. The precise form of the consensus term depends on the gossip mechanism used. The resulting bound is standard in spirit and is comparable to existing results such as \cite{koloskova2020unified}. However, the analysis highlights the benefit of the outer momentum mechanism: it reduces the number of communication rounds required to achieve a given level of consensus.%

\vspace{-7pt}
\section{Numerical Experiments}\vspace{-7pt}\label{sec:numerical-experiments}
We evaluate \ours in a distributed LLM pre-training setting. We compare against \ddp~\citep{li2020pytorchdistributed} and \diloco~\citep{douillard2023diloco} as standard distributed baselines. As decentralized baseline, we compare with \DAdam~\citep{wang2024promise}, a recent approach that was able to scale to LLMs and incorporate adaptive optimizers effectively. It uses an Adam-based gossip method, and AdamW is a standard~\citep{groeneveld2024olmo,grattafiori2024llama} local optimizer for transformer pre-training. Moreover, \DAdam supports overlapping communication with computation, has convergence guarantees and was shown to achieve strong generalization performance in practical multi-node settings, including transformer training. %

We pretrain Llama-3-style decoder-only Transformer models~\citep{grattafiori2024llama} on FineWeb~\citep{penedo2024fineweb}. Following \citet{hoffmann2022training}, we train a 134M-parameter model on 2.7B tokens and a 551M-parameter model on 11B tokens. We use the LLaMA-2 tokenizer~\citep{touvron2023llama2}, 2048-token sequences, a global batch size of 2M tokens, and \(n \in \{8,16,32\}\) workers unless stated otherwise. All methods use AdamW~\citep{loshchilov2018decoupled} locally, with weight decay \(0.1\), \((\beta_1,\beta_2)=(0.9,0.95)\), a tuned peak learning rate, 10\% warmup from 1\% of the peak, and cosine decay. For \(H>1\), \diloco and \ours perform \(H\) local updates before an outer step: \diloco computes the outer pseudo-gradient by \allreduce, whereas \ours communicates over the selected topology. Both use independently tuned Nesterov momentum~\citep{sutskever2013on} for the outer optimizer.  Our sparse variants use randomized peer-to-peer communication: \ours-1-Peer forms disjoint random symmetric pairings, and \ours-2-Peer connects adjacent workers in a randomly induced ring. See Appendix~\ref{app:exp:protocol} for more details.

\paragraph{\versionD.}
\DAdam is directly comparable to our protocol only when \(H=1\), since the original algorithm mixes neighboring parameters after every local optimizer step. For \(H>1\), delaying this mixing changes the algorithm rather than only reducing communication frequency. We therefore introduce a variant of DAdam which we denote \versionD as a controlled adaptation for infrequent communication: it keeps the same \(H\)-step inner loop and communication schedule as \ours, but replaces the outer pseudo-gradient with a \DAdam-style disagreement term computed before the local block. Appendix~\ref{app:versionD} gives the precise update and explains how we adapted \DAdam for comparison. %

\begin{table*}[t]
\centering
\caption{\textbf{Decentralized training on the 134M model.} Validation loss for one and multiple (\(H=30\)) local steps decentralized training across different topologies and replica counts. All runs use a fixed training budget of 2.68B tokens.}
\label{tab:small-decentralized-results}

\begin{subtable}[t]{0.33\textwidth}
\centering
\caption{\textbf{One-step}}
\label{tab:small-1step}
{\fontsize{7}{7.5}\selectfont
\begin{tabular}{@{}lcc@{}}
\toprule
\multicolumn{2}{l}{AdamW DDP (\textit{ref.})} & 3.18 \\
\midrule
\textbf{Topology} &
\textbf{\ours} &
\textbf{\DAdam} \\
\midrule
ring
& 3.36 & \textbf{3.33} \\
complete
& 3.29 & \textbf{3.25} \\
1-Peer
& \textbf{3.24} & 3.29 \\
2-Peer
& \textbf{3.22} & 3.28 \\
\bottomrule
\end{tabular}}
\end{subtable}
\hfill
\begin{subtable}[t]{0.65\textwidth}
\centering
\caption{\textbf{Multi-step}}
\label{tab:small-30steps}
\resizebox{\linewidth}{!}{%
\begin{tabular}{@{}lcccccc@{}}
\toprule
& \multicolumn{2}{c}{\textbf{8 replicas}}
& \multicolumn{2}{c}{\textbf{16 replicas}}
& \multicolumn{2}{c}{\textbf{32 replicas}} \\
\cmidrule(lr){2-3}
\cmidrule(lr){4-5}
\cmidrule(lr){6-7}
\diloco \textit{(ref.)}
& \multicolumn{2}{c}{3.30}
& \multicolumn{2}{c}{3.40}
& \multicolumn{2}{c}{3.53} \\
\midrule
\textbf{Topology}
& \textbf{\ours} & \textbf{Loc.-\DAdam}
& \textbf{\ours} & \textbf{Loc.-\DAdam}
& \textbf{\ours} & \textbf{Loc.-\DAdam} \\
\midrule
ring
& \textbf{3.32} & 3.51
& \textbf{3.45} & 3.69
& \textbf{3.81} & 4.00 \\
complete
& \textbf{3.30} & 3.47
& \textbf{3.40} & 3.65
& \textbf{3.53} & 3.86 \\
1-Peer
& \textbf{3.34} & 3.51
& \textbf{3.49} & 3.66
& \textbf{3.66} & 3.88 \\
2-Peer
& \textbf{3.32} & 3.48
& \textbf{3.45} & 3.66
& \textbf{3.62} & 3.86 \\
\bottomrule
\end{tabular}
}
\end{subtable}
\end{table*}

\begin{wraptable}{r}{0.5\textwidth}
\vspace{-1.3em}
    \centering
    \caption{\textbf{One-step Decentralized Methods.} Validation loss on 16 replicas for the 551M-parameter LLM after training on 11.02B tokens.}
    \label{tab:medium-1step-scaling}
    {\fontsize{11}{12.5}\selectfont
    \begin{tabular}{llc}
    \toprule
    \textbf{Method} &
    \textbf{Topology} &
    \textbf{Val Loss} \\
    \midrule
    AdamW DDP (\textit{ref.}) & --
    & 2.64 \\
    
    \cdashline{1-3}

    DAdam & ring
    & 2.73 \\
    
    \ours & ring
    & \textbf{2.70} \\
    \cdashline{1-3}
       DAdam & complete
    & 2.70 \\
    \ours-1-Peer & complete
    & 2.71 \\
    
    \ours-2-Peer & complete
    & \textbf{2.69} \\
    \bottomrule
    \end{tabular}
    }
\end{wraptable}

\subsection{Results with Homogeneous Workers}

\paragraph{Single-step Evaluation.} \label{sec:1-step-eval}
We first evaluate the methods in the single-step regime, where each communication round follows one local optimizer update. This setting allows us to directly compare to \citet{wang2024promise}. Our results are shown in \Cref{tab:small-1step} and \Cref{tab:medium-1step-scaling}.
Among decentralized methods, the strongest variant is \ours-2-Peer, which improves over complete-graph \DAdam, despite using substantially sparser communication.
This comparison is conservative in communication terms: complete-graph \DAdam communicates with all workers at every iteration, whereas \ours-2-Peer communicates with only two randomly selected peers per worker per round. 
A more communication-matched comparison is given by \DAdam-2-Peer, where both methods use the same randomized pattern. In this setting, \ours-2-Peer increases the gap, improving validation loss by 0.06.
These results suggest that randomized sparse communication in GASLoC can mitigate the degradation observed for complete-graph updates, while also avoiding global synchronization. GASLoC also shows competitive results on the challenging ring graph. At a larger scale, we observe even better trend in terms of GASLoC performance, with \ours-2-Peer significantly outperforming
DAdam as shown in \Cref{tab:medium-1step-scaling} in ring and complete topologies.

\paragraph{Multi-step regime.}
The \(H=30\) regime evaluates the setting in which the communication savings of local methods are most relevant. Each worker performs many optimizer updates before the outer synchronization step, so the communication step must correct both consensus error and the drift accumulated by local training. In this regime, \diloco is a strong reference point: it uses the same local computation pattern as \ours, but its outer step still relies on a complete \allreduce. The central question is therefore not whether sparse communication improves over complete averaging in isolation, but whether it can preserve comparable validation loss while removing the global synchronization requirement. \Cref{tab:small-30steps} shows that \ours with the complete graph matches \diloco across worker counts, as it recovers the same communication structure. We observe, on the other hand, that \versionD does not scale well for multiple steps, obtaining significantly worse results than DiLoCo. Furthermore, unlike \ours, it does not gain benefits from randomized peer communication. On the other hand we observe that \ours achieves performance close to DiLoCo even in the 1-Peer case where communication is both reduced and more fault tolerant. Note that on the ring graph, performance nearly matches that of the complete graph. In the larger scale setting in \Cref{tab:medium-30steps}, 
\ours-2-Peer can match \diloco and confirm the degradation of DAdam with local steps as well as the superior performance of \ours in this setting. 

\begin{table}[]
    \centering
    \caption{\textbf{Multi-step Decentralized Methods.} Validation loss for the 551M model, communicating every 30 steps using 8 and 16 replicas. All runs use a fixed training budget of 11.02B tokens.
    }
    \vspace{5pt}
    \label{tab:medium-30steps}
    {\fontsize{11}{12.5}\selectfont
    \begin{tabular}{llcc}
    \toprule
    \multirow{2}{*}{\textbf{Method}} &
    \multirow{2}{*}{\textbf{Topology}} &
    \multicolumn{2}{c}{\textbf{Val Loss}} \\
    \cline{3-4}
    &&
    \textbf{8 replicas} &
    \textbf{16 replicas} \\
    \midrule
    DiLoCo (\textit{ref.}) & --
    & 2.64
    & 2.67 \\
    
    \cdashline{1-4}

    \versionD & ring
    & 2.92
    & 3.07 \\
\ours & ring
    & \textbf{2.68}
    & \textbf{2.79} \\
\cdashline{1-4}
    \versionD & complete
    & 2.84
    & 2.97 \\

    \ours-1-Peer & complete
    & 2.67
    & 2.81 \\
    
    \ours-2-Peer & complete
    & \textbf{2.64}
    & \textbf{2.72} \\
    \bottomrule
\end{tabular}}
\end{table}

\subsection{Heterogeneous Communication Setting}

\begin{figure}[t]
    \centering
    \begin{subfigure}[b]{0.48\textwidth}
        \centering
        \includegraphics[width=\textwidth]{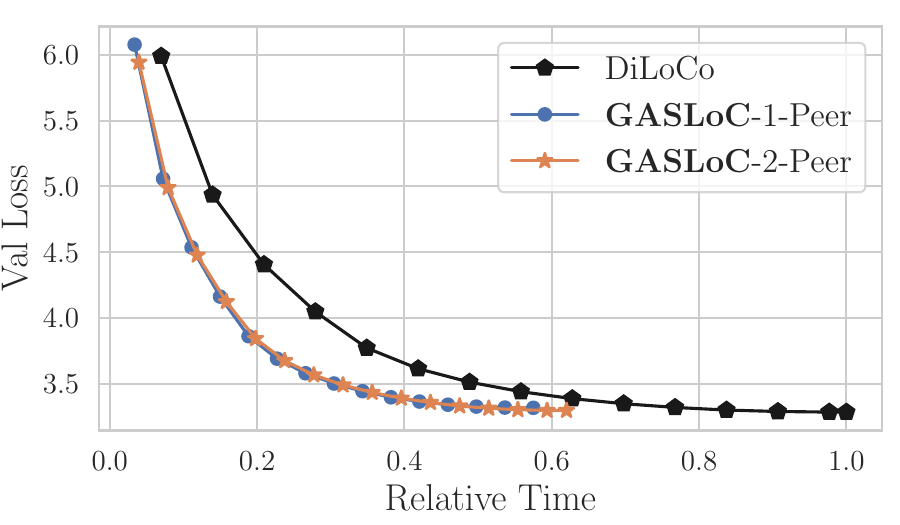}
        \caption{Straggler with 10\% of the bandwidth}
        \label{fig:heterogeneous-setting-10}
    \end{subfigure}
    \hfill %
    \begin{subfigure}[b]{0.48\textwidth}
        \centering
        \includegraphics[width=\textwidth]{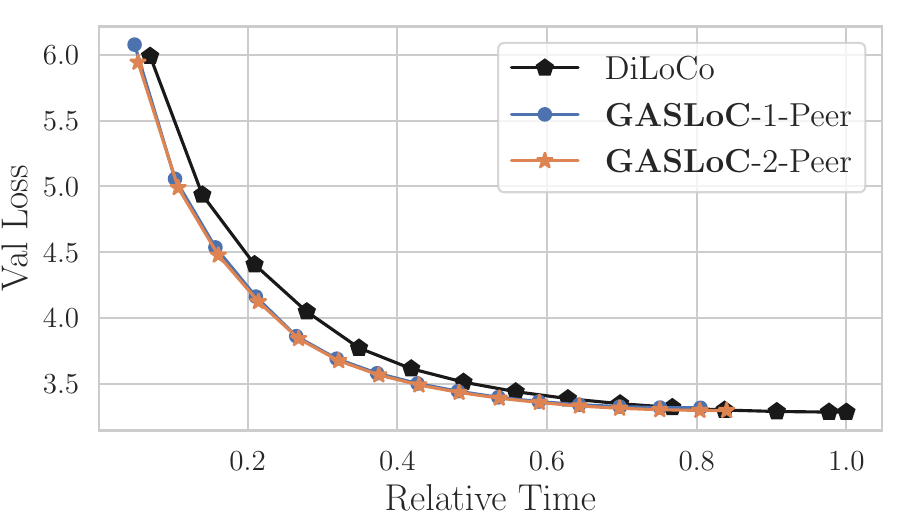}
        \caption{Straggler with 20\% of the bandwidth}
        \label{fig:heterogeneous-setting-20}
    \end{subfigure}
    \caption{\textbf{Robustness to bandwidth stragglers.}
Validation loss versus relative wall-clock time when one worker has reduced communication bandwidth.
\ours adapts to the straggler by reducing its local computation while keeping the non-straggling workers at \(H=30\).
At 10\% bandwidth (a), the straggler performs \(H_i=15\) steps for \ours-1-Peer and \(H_i=1\) for \ours-2-Peer.
At 20\% bandwidth (b), the lower communication cost allows the straggler to perform the full \(H_i=30\) steps for \ours-1-Peer and \(H_i=15\) for \ours-2-Peer.
Sparse peer-to-peer communication \emph{substantially improves} time-to-loss relative to \diloco, especially under severe straggling.}
    \label{fig:heterogeneous-setting}
\end{figure}

We next evaluate the effect of communication heterogeneity for the 134M-parameter model under the same fixed training budget set by Chinchilla scaling laws~\citep{hoffmann2022training}. As discussed, unlike DiLoCo, \ours is able to vary the number of local steps while still allowing synchronization to occur at the same time. In this experiment, all methods process the same total number of tokens, but one worker's communication bandwidth is reduced to either 10\% or 20\% of the bandwidth available to the other workers. We simulate a relatively common over-the-internet bandwidth of 1 Gbps across nodes and 100 or 200 Mbps for the inbound and outbound links on the straggler node. We use 8 nodes, each with an A100 GPU.  Following common \diloco implementations~\citep{therien2025muloco}, we model its outer synchronization as an \allreduce over the outer delta. This implementation is efficient on homogeneous interconnects, but every outer step still requires a global collective and therefore is delayed by the low-bandwidth worker. In contrast, \ours--1-Peer and \ours--2-Peer replace this global synchronization with sparse randomized peer exchanges and allow the number of local steps to vary across workers, so that the straggler's longer communication time is compensated by reduced local computation before synchronization. \Cref{fig:heterogeneous-setting} reports the validation loss as training progresses over normalized wall-clock time, using \diloco as reference. We observe that under this practical scenario a significant wall-clock time advantage can be obtained over DiLoCo.

\begin{wrapfigure}[14]{r}{0.48\textwidth}
    \centering
    \vspace{-38pt}
    \includegraphics[width=\linewidth]{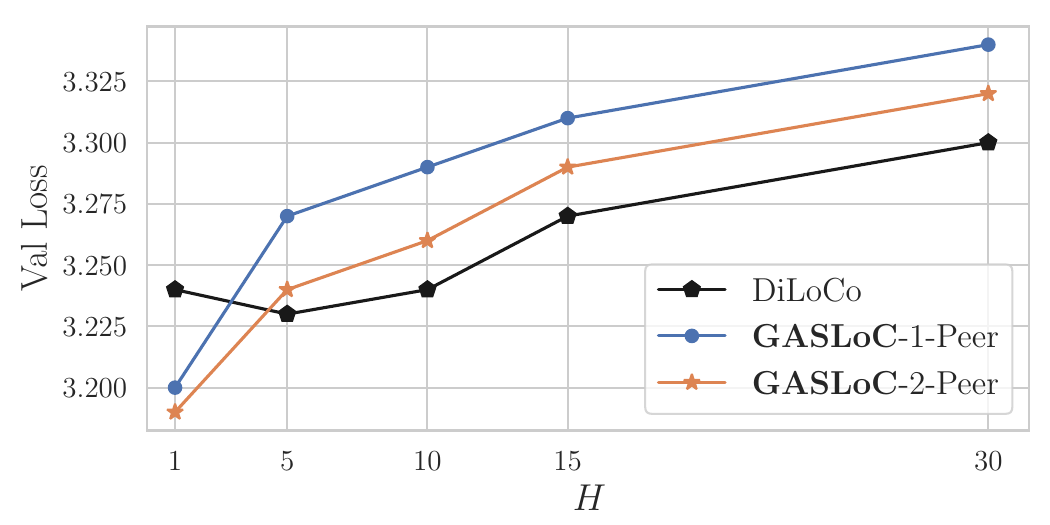}
    \caption{Final validation loss for a local-step sweep on the 134M model with 8 workers. We compare \diloco and sparse \ours variants with one or two randomized peer exchanges per outer step. Sparse variants follow the same qualitative trend as \diloco as \(H\) increases.}
    \label{fig:local-step-sweep}
\end{wrapfigure}

\subsection{Sensitivity to Communication Frequency}

We also analyze the sensitivity of \ours to the number of local steps, in particular in the 1-Peer and 2-Peer settings. \Cref{fig:local-step-sweep} evaluates whether sparse randomized communication preserves the same dependence on communication frequency as \diloco. We first observe that \ours-1-Peer and \ours-2-Peer variants outperform DiLoCo at low \(H\). This is consistent with the results in \citet{defazio2025smoothing}, showing that at very low H \diloco performance can degrade. For \(H>5\), the trend is consistent in that \ours-2-Peer and 1-Peer slightly underperform DiLoCo (while having communication benefits) but have similar scaling properties with respect to inner steps.

\section{Conclusion}
We introduced \ours, a decentralized pre-training method that combines local optimizer steps, sparse randomized peer communication, and an outer optimizer. The method unifies \diloco and gossip-based training: complete communication recovers the \diloco update, while sparse peer communication replaces global collectives with local exchanges and connects the outer momentum mechanism to communication acceleration. We showed, both theoretically and empirically, that this formulation preserves the main behavior of \diloco in the homogeneous setting while removing the dependence on global synchronization. In a simulated bandwidth-straggler setting, \ours reaches a comparable validation loss in a shorter wall-clock time by combining sparse communication and worker-specific local step counts. These results provide, to our knowledge, the first evidence that \localsgd-style LLM pre-training can be effectively combined with decentralized peer-to-peer communication when global synchronization is costly.

\section*{Acknowledgements}
This work was supported by PEPR IA on grant
SHARP ANR-23-PEIA-0008 and PEPR NUMPEX on grant  DAIMOS ANR-25-EXNU-0002. Part of this work was
granted access to the JeanZay HPC/AI resources of IDRIS under the allocation AD011015884R1, AD011017481, AD011016641, AD011017661 and AD011017766. E.B. and P.C. acknowledge funding from NSERC Discovery and FRQNT New Scholar.

\bibliographystyle{plainnat}
\bibliography{sample}
\newpage
\appendix

\section{Proof of \Cref{prop:cvg}}
\label{app:math:convergence}
\begin{proposition}
Let \(x^\star\in\arg\min f\), and suppose that \(x^\star\) is an unconstrained
minimizer, so that \(\nabla f(x^\star)=0\). Suppose that, for every \(\xi\),
\(F(\cdot;\xi)\) is \(L\)-smooth. Assume moreover that the stochastic gradients
are unbiased and have variance bounded by \(\sigma^2\), namely
\[
\mathbb{E}_\xi[\nabla F(x;\xi)]=\nabla f(x),
\qquad
\mathbb{E}_\xi\bigl[
\|\nabla F(x;\xi)-\nabla f(x)\|^2
\bigr]\le \sigma^2 .
\]
Assume that
$0<\alpha<1,0<\beta\le \frac1{8L},$
and that the outer step size satisfies
$
0<\eta
\le
\frac{\alpha}{3\sqrt{3}\beta\rho L}.$ The constant \(\rho\) depends on the communication scheme as follows:
\begin{enumerate}
    \item \textbf{Standard gossip with spectral gap \(\chi\):} we take
    \[
    \rho=\chi,
    \qquad
    \gamma=0.
    \]

    \item \textbf{Accelerated gossip with spectral gap $\chi$ via outer momentum:} we take
    \[
    \rho=\sqrt{\chi},
    \qquad
    \gamma=
    \left(
    1-\sqrt{\frac{\alpha}{\chi}}
    \right)^2 .
    \]

    \item \textbf{\(k\)-Peer gossip:} we take
    \[
    \rho =
    \begin{cases}
    \displaystyle \frac{n}{n-1},
    & \text{for randomized $1$-Peer communication}, \\[1.2em]
    \displaystyle
    \frac{\alpha}{
    1-
    \sqrt{
    1-\frac{n}{n-1}\alpha
    +
    \frac{3n}{8(n-1)}\alpha^2
    }},
    & \text{for randomized $2$-Peer communication},
    \end{cases}
    \qquad
    \gamma=0.
\]
\end{enumerate}
If the initialization satisfies $x_0^1=\cdots=x_0^n=x_0,$ then, for every \(T\ge 1\),
\[
\frac1T
\sum_{t=0}^{T-1}
\mathbb{E}\bigl[\|\nabla f(\bar x_t)\|^2\bigr]
\le
\frac{4(f(x_0)-f(x^\star))}{T\eta\beta}
+
\frac{2L\eta\beta\sigma^2}{n^2}
\sum_{i=1}^n H_i
+
60\beta^2L^2\sigma^2 .
\]
\end{proposition}

\[
\bar x \triangleq \frac1n\sum_{i=1}^n x^i
\]
denote the network average.
We introduce the Bregman divergence associated with \(f\):
\begin{equation}
d_f(x,y)\triangleq f(x)-f(y)-\langle \nabla f(y),\,x-y\rangle.
\end{equation}
The following standard identities will be used repeatedly:
\begin{align}
d_F(x,z)
&=
d_F(x,y)+d_F(y,z)+\langle x-y,\nabla F(y)-\nabla f(z)\rangle, \\
d_F(x,y)
&\le \frac{L}{2}\|x-y\|^2, \\
d_F(x,x^\star)
&\ge 0.
\end{align}
Note that the above also holds for $f=\mathbb{E}[F]$.

For the inner loop, recall that for each node \(i\),
\begin{equation}
x_{t,k+1}^i = x_{t,k}^i - \beta \nabla F(x_{t,k}^i;\xi^i_{t,k}),
\qquad k=0,\dots,H_i-1,
\end{equation}
with \(x_{t,0}^i=x_t^i\). We define the accumulated local update
\begin{equation}
g_t^i \triangleq \beta \sum_{k=0}^{H_i-1} \nabla F(x_{t,k}^i;\xi_{t,k}^i),
\end{equation}
so that
\begin{equation}
x_{t,H_i}^i = x_t^i - g_t^i.
\end{equation}
Stacking the local vectors, we write \(g_t=(g_t^1,\dots,g_t^n)\), and the outer update becomes
\begin{equation}
x_{t+1}
=
x_t+\eta g_t-\alpha \Lambda(x_t+\eta g_t).
\end{equation}

Let
\[
\bar x_t \triangleq \frac{1}{n}\sum_{i=1}^n x_t^i,
\qquad
\bar g_t \triangleq \frac{1}{n}\sum_{i=1}^n g_t^i.
\]
Since \(\Lambda \mathbf 1 = 0\), averaging the outer update yields
\begin{equation}
\bar x_{t+1} = \bar x_t + \eta \bar g_t.
\end{equation}

Applying the three-point identity for \(d_f\) with
\[
x=\bar x_{t+1},\qquad y=\bar x_t,\qquad z=x^\star,
\]
we obtain
\begin{align}
d_f(\bar x_{t+1},x^\star)-d_f(\bar x_t,x^\star)
&=
d_f(\bar x_{t+1},\bar x_t)
+
\left\langle
\bar x_{t+1}-\bar x_t,
\nabla f(\bar x_t)-\nabla f(x^\star)
\right\rangle \\
&=
d_f(\bar x_t+\eta \bar g_t,\bar x_t)
+
\eta
\left\langle
\bar g_t,
\nabla f(\bar x_t)-\nabla f(x^\star)
\right\rangle .
\end{align}
If \(x^\star\) is an unconstrained minimizer, then \(\nabla f(x^\star)=0\), and therefore
\begin{equation}
d_f(\bar x_{t+1},x^\star)-d_f(\bar x_t,x^\star)
=
d_f(\bar x_t+\eta \bar g_t,\bar x_t)
+
\eta \langle \bar g_t,\nabla f(\bar x_t)\rangle.
\end{equation}
Using \(L\)-smoothness,
\begin{equation}
d_f(\bar x_t+\eta \bar g_t,\bar x_t)
\le
\frac{L\eta^2}{2}\|\bar g_t\|^2.
\end{equation}
Then, by the bias--variance decomposition,
\begin{align}
\mathbb E\|\bar g_t\|^2
&=
\left\|\mathbb E_t[\bar g_t]\right\|^2
+
\mathbb E
\left\|
\bar g_t-\mathbb E[\bar g_t]
\right\|^2 \\
&=
\left\|
\frac1n\sum_{i=1}^n \mathbb E[g_t^i]
\right\|^2
+
\frac1{n^2}
\sum_{i=1}^n
\mathbb E
\left\|
g_t^i-\mathbb E[g_t^i]
\right\|^2 \\
&\le\beta^2
\left\|
\frac1n\sum_{i=1}^n \frac 1{H_i}\sum_{k=0}^{H_i-1}\nabla f(x_{t,k}^i)
\right\|^2
+
\frac{\beta^2\sigma^2}{n^2}
\sum_{i=1}^n H_i .
\end{align}

Now, we will use that, taking the expectations:

\begin{align}
   \mathbb{E} \langle \bar g_t,\nabla f(\bar x_t)\rangle =\frac 1n\sum_{i=1}^n\mathbb{E}\langle g^i_t,\nabla f(\bar x_t)\rangle
\end{align}

Here, using Lemma \ref{lemma:gradient-mismatch}, we have:
\begin{equation}
\begin{aligned}
\mathbb E\!\left[
\left\langle g_t^i,\nabla f(\bar x_t)\right\rangle
\right]
&\le
-\beta\left(
\frac12-\frac32\beta^2L^2e^2
\right)
\mathbb E\|\nabla f(\bar x_t)\|^2 -\frac{\beta}{2}\mathbb E\|\frac 1{H_i}\sum_{k=0}^{H_i-1}\nabla f(x_{t,k}^i)\|^2\\
&\quad+
\frac{3}{2}\beta L^2e^2
\mathbb E\|x_t^i-\bar x_t\|^2
+
\frac{3}{2}\beta^3L^2e^2\sigma^2 .
\end{aligned}
\end{equation}

Introduce $G^i_t=\frac 1{H_i}\sum_{k=0}^{H_i-1}\nabla f(x_{t,k}^i)$. Combining the last two inequalities yields
\begin{align}
&\mathbb E\!\left[
d_f(\bar x_{t+1},x^\star)
-
d_f(\bar x_t,x^\star)
\right] \notag \\
&\le
-
\eta\beta
\left(
\frac12-\frac32\beta^2L^2e^2
\right)
\mathbb E\|\nabla f(\bar x_t)\|^2 \notag \\
&\quad
-
\frac{\beta \eta}{2n}
\sum_{i=1}^n
\mathbb E
\left\|
G^i_t
\right\|^2 +\frac {\eta^2L\beta^2}{2}\mathbb E
\left\|\frac 1n\sum_{i=1}^n
G^i_t
\right\|^2\notag \\
&\quad
+
\frac{3\eta\beta L^2e^2}{2n}
\sum_{i=1}^n
\mathbb E\|x_t^i-\bar x_t\|^2 \notag \\
&\quad
+
\frac{L\eta^2\beta^2\sigma^2}{2n^2}
\sum_{i=1}^n H_i
+
\frac{3}{2}\eta\beta^3L^2e^2\sigma^2 .
\end{align}

Next, we use the communication contraction bound given by Lemma \ref{lemma:gossip}.
\[
\Delta V_\pi
\le
-\frac{\alpha}{\rho}\|\pi x_t\|^2
+
\frac{\rho\eta^2\beta^2}{\alpha}\|\pi G_t\|^2.
\]
Introduce the Lyapunov function
\[
\phi
\triangleq
d_f(\bar x,x^\star)+\Omega V_{\pi}(x).
\]
Combining the previous inequalities gives
\begin{align}
\mathbb{E}[\phi_{t+1}-\phi_t]
&\le
-
\eta\beta
\left(
\frac12-\frac32\beta^2L^2e^2
\right)
\mathbb E\|\nabla f(\bar x_t)\|^2 \notag \\
&\quad
-
\frac{\beta \eta}{2n}
\sum_{i=1}^n
\mathbb E
\left\|
G_t^i
\right\|^2
+
\frac{\eta^2L\beta^2}{2}
\mathbb E
\left\|
\frac1n\sum_{i=1}^n G_t^i
\right\|^2 \notag \\
&\quad
+
\frac{3\eta\beta L^2e^2}{2n}
\sum_{i=1}^n
\mathbb E\|x_t^i-\bar x_t\|^2 \notag \\
&\quad
+
\frac{L\eta^2\beta^2\sigma^2}{2n^2}
\sum_{i=1}^n H_i
+
\frac32\eta\beta^3L^2e^2\sigma^2 \notag \\
&\quad
+
\Omega
\left(
-\frac{\alpha}{\rho}\mathbb E\|\pi x_t\|^2
+
\frac{\rho\eta^2\beta^2}{\alpha}\mathbb E\|\pi G_t\|^2
\right).
\end{align}

Using
\[
\sum_{i=1}^n
\mathbb E\|G_t^i\|^2
=
\mathbb E\|\pi G_t\|^2
+
n\mathbb E\|\bar G_t\|^2,
\qquad
\bar G_t
\triangleq
\frac1n\sum_{i=1}^n G_t^i,
\]
we recombine the inequalities and obtain
\begin{align}
\mathbb{E}[\phi_{t+1}-\phi_t]
&\le
-
\eta\beta
\left(
\frac12-\frac32\beta^2L^2e^2
\right)
\mathbb E\|\nabla f(\bar x_t)\|^2 \notag \\
&\quad
-
\left(
\frac{\beta\eta}{2}
-
\frac{L\beta^2\eta^2}{2}
\right)
\mathbb E\|\bar G_t\|^2 \notag \\
&\quad
-
\left(
\frac{\beta\eta}{2n}
-
\Omega\frac{\rho\eta^2\beta^2}{\alpha}
\right)
\mathbb E\|\pi G_t\|^2 \notag \\
&\quad
-
\left(
\Omega\frac{\alpha}{\rho}
-
\frac{3\eta\beta L^2e^2}{2n}
\right)
\mathbb E\|\pi x_t\|^2 \notag \\
&\quad
+
\frac{L\eta^2\beta^2\sigma^2}{2n^2}
\sum_{i=1}^n H_i
+
\frac32\eta\beta^3L^2e^2\sigma^2 .
\end{align}

Here, we gonna pick $\Omega
=
\frac{3\rho\eta\beta L^2e^2}{2\alpha n}$

so that:
\begin{align}
\left(
-\frac{\beta\eta}{2n}
+
\Omega\frac{\beta^2\rho\eta^2}{\alpha}
\right)
&=
-\frac{\beta\eta}{2n}
+
\frac{3\beta^2\rho^2\eta^3\beta L^2e^2}{2\alpha^2 n} \\
&=
\frac{\beta\eta}{2n}
\left(
-1
+
\frac{3\rho^2\beta^2\eta^2L^2e^2}{\alpha^2}
\right).
\end{align}

Thus it is necessary so that $\frac{3\rho^2\beta^2\eta^2L^2e^2}{\alpha^2}\leq 1$, $\frac32\beta^2L^2e^2\leq \frac 14$ and $L\beta\eta\leq 1$

Then, we obtain
\[
\mathbb{E}[\phi_{t+1}-\phi_t]
\le
-\frac{\eta\beta}{4}
\mathbb E\|\nabla f(\bar x_t)\|^2
+
\frac{L\eta^2\beta^2\sigma^2}{2n^2}
\sum_{i=1}^n H_i
+
\frac32\eta\beta^3L^2e^2\sigma^2 .
\]

Summing the above inequality from \(t=0\) to \(T-1\), we get
\begin{align}
\sum_{t=0}^{T-1}
\mathbb{E}[\phi_{t+1}-\phi_t]
&\le
-\frac{\eta\beta}{4}
\sum_{t=0}^{T-1}
\mathbb E\|\nabla f(\bar x_t)\|^2
+
\sum_{t=0}^{T-1}
\left(
\frac{L\eta^2\beta^2\sigma^2}{2n^2}
\sum_{i=1}^n H_i
+
\frac32\eta\beta^3L^2e^2\sigma^2
\right) \\
&=
-\frac{\eta\beta}{4}
\sum_{t=0}^{T-1}
\mathbb E\|\nabla f(\bar x_t)\|^2
+
T
\left(
\frac{L\eta^2\beta^2\sigma^2}{2n^2}
\sum_{i=1}^n H_i
+
\frac32\eta\beta^3L^2e^2\sigma^2
\right).
\end{align}

Since the left-hand side telescopes, we have
\[
\sum_{t=0}^{T-1}
\mathbb{E}[\phi_{t+1}-\phi_t]
=
\mathbb E[\phi_T]-\phi_0.
\]
Therefore,
\begin{align}
\mathbb E[\phi_T]-\phi_0
&\le
-\frac{\eta\beta}{4}
\sum_{t=0}^{T-1}
\mathbb E\|\nabla f(\bar x_t)\|^2
+
T
\left(
\frac{L\eta^2\beta^2\sigma^2}{2n^2}
\sum_{i=1}^n H_i
+
\frac32\eta\beta^3L^2e^2\sigma^2
\right).
\end{align}

Since \(\phi_T\ge 0\), it follows that
\[
-\phi_0
\le
\mathbb E[\phi_T]-\phi_0.
\]
Combining this with the previous inequality gives
\begin{align}
-\phi_0
&\le
-\frac{\eta\beta}{4}
\sum_{t=0}^{T-1}
\mathbb E\|\nabla f(\bar x_t)\|^2
+
T
\left(
\frac{L\eta^2\beta^2\sigma^2}{2n^2}
\sum_{i=1}^n H_i
+
\frac32\eta\beta^3L^2e^2\sigma^2
\right).
\end{align}
Rearranging, we obtain
\begin{align}
\frac{\eta\beta}{4}
\sum_{t=0}^{T-1}
\mathbb E\|\nabla f(\bar x_t)\|^2
&\le
\phi_0
+
T
\left(
\frac{L\eta^2\beta^2\sigma^2}{2n^2}
\sum_{i=1}^n H_i
+
\frac32\eta\beta^3L^2e^2\sigma^2
\right).
\end{align}

Dividing both sides by \(T\eta\beta/4\), we get
\begin{align}
\frac1T
\sum_{t=0}^{T-1}
\mathbb E\|\nabla f(\bar x_t)\|^2
&\le
\frac{4\phi_0}{T\eta\beta}
+
\frac{4}{\eta\beta}
\left(
\frac{L\eta^2\beta^2\sigma^2}{2n^2}
\sum_{i=1}^n H_i
+
\frac32\eta\beta^3L^2e^2\sigma^2
\right) \\
&=
\frac{4(f(x_0)-f(x^*))}{T\eta\beta}
+
\frac{2L\eta\beta\sigma^2}{n^2}
\sum_{i=1}^n H_i
+
6\beta^2L^2e^2\sigma^2 .
\end{align}

\begin{lemma}[Gossip lemma]\label{lemma:gossip}
\textbf{(i) Standard gossip.}
Assume \(\gamma=0\). Then
\[
\|\pi x_{t+1}\|^2-\|\pi x_t\|^2
\le
-\frac{\alpha}{\chi}\|\pi x_t\|^2
+
\frac{\chi}{\alpha}\|\pi y_t\|^2 .
\]

\medskip

\noindent
\textbf{(ii) Random \(k\)-Peer gossip.}
Assume \(\gamma=0\). For the randomized \(1\)-Peer scheme, the bound holds with
\[
    \chi_1=\frac{n}{n-1},
\]
namely
\[
\mathbb E_t\!\left[\|\pi x_{t+1}\|^2\right]
-
\|\pi x_t\|^2
\le
-\frac{\alpha}{\chi_1}\|\pi x_t\|^2
+
\frac{\chi_1}{\alpha}\|\pi y_t\|^2 .
\]
For the randomized \(2\)-Peer scheme, the bound holds with
\[
    \chi_2
    =
    \frac{\alpha}{
    1-
    \sqrt{
    1-\frac{n}{n-1}\alpha
    +
    \frac{3n}{8(n-1)}\alpha^2
    }}
\]
as
\[
\mathbb E_t\!\left[\|\pi x_{t+1}\|^2\right]
-
\|\pi x_t\|^2
\le
-\frac{\alpha}{\chi_2}\|\pi x_t\|^2
+
\frac{\chi_2}{\alpha}\|\pi y_t\|^2 .
\]
\medskip

\noindent
\noindent
\textbf{(iii) Accelerated gossip.}
Assume
\[
\gamma=\left(1-\sqrt{\frac{\alpha}{\chi}}\right)^2,
\qquad 0<\alpha\le 1,
\]
and define
\[
X_t:=
\begin{pmatrix}
\pi x_t\\
\pi x_{t-1}
\end{pmatrix}.
\]
Then,
\[
\|X_{t+1}\|^2-\|X_t\|^2
\le
-\frac{\alpha}{\sqrt{\chi}}\|X_t\|^2
+
\frac{\sqrt{\chi}}{\alpha}\|\pi y_t\|^2\leq -\frac{\alpha}{\sqrt{\chi}}\|\pi x_t\|^2
+
\frac{\sqrt{\chi}}{\alpha}\|\pi y_t\|^2.
\]
\end{lemma} %
\begin{proof}

We prove the three claims separately.

\medskip

\noindent
\textbf{Proof of (i).}
Since \(\Lambda\mathbf 1=0\), we have
\[
\pi\Lambda=\Lambda\pi=\Lambda.
\]
Thus, when \(\gamma=0\),
\[
\pi x_{t+1}
=
(I-\alpha\Lambda)(\pi x_t+\pi y_t).
\]
On the disagreement subspace, the eigenvalues of \(I-\alpha\Lambda\) are
\[
1-\alpha\lambda_r(\Lambda),
\qquad r=2,\dots,n.
\]
Since \(0<\alpha\le 1\) and
\[
\lambda_r(\Lambda)\in\left[\frac1\chi,1\right],
\]
we get
\[
\|(I-\alpha\Lambda)u\|
\le
\left(1-\frac{\alpha}{\chi}\right)\|u\|
\]
for all \(u\perp \mathbf 1\). Hence
\[
\|\pi x_{t+1}\|^2
\le
\left(1-\frac{\alpha}{\chi}\right)^2
\|\pi x_t+\pi y_t\|^2.
\]
Using Young's inequality,
\[
\|a+b\|^2
\le
(1+\nu)\|a\|^2+
\left(1+\frac1\nu\right)\|b\|^2,
\]
with
\[
\nu=\frac{1}{1-\alpha/\chi}-1,
\]
we obtain
\[
\|\pi x_{t+1}\|^2-\|\pi x_t\|^2
\le
-\frac{\alpha}{\chi}\|\pi x_t\|^2
+
\frac{\chi}{\alpha}
\left(1-\frac{\alpha}{\chi}\right)^2
\|\pi y_t\|^2.
\]
Since \(\left(1-\frac{\alpha}{\chi}\right)^2\le 1\), this gives
\[
\|\pi x_{t+1}\|^2-\|\pi x_t\|^2
\le
-\frac{\alpha}{\chi}\|\pi x_t\|^2
+
\frac{\chi}{\alpha}\|\pi y_t\|^2.
\]

\medskip

\noindent
\noindent

\textbf{Proof of (ii).}
We have
\begin{align}
\mathbb E_t\!\left[\|\pi x_{t+1}\|^2\right]
&=
\mathbb E_t\!\left[
\|(I-\alpha\Lambda_t)\pi(x_t+y_t)\|^2
\right] \\
&=
(x_t+y_t)^\top
\mathbb E_t\!\left[
\pi-2\alpha\Lambda_t+\alpha^2\Lambda_t^2
\right]
(x_t+y_t).
\end{align}

For the $1$-Peer scheme, since
\[
    \mathbb E[\Lambda_t]
    =
    \mathbb E[\Lambda_t^2]
    =
    \frac{n}{2(n-1)}\pi,
\]
we obtain
\[
\mathbb E_t\!\left[\|\pi x_{t+1}\|^2\right]
\le
\left(
1-2\frac{n}{n-1}\alpha
+
\frac{n}{n-1}\alpha^2
\right)
\|\pi x_t+\pi y_t\|^2.
\]
Moreover,
\[
1-2\frac{n}{n-1}\alpha
+
\frac{n}{n-1}\alpha^2
\le
\left(
1-\frac{n}{n-1}\alpha
\right)^2.
\]
Hence
\[
\mathbb E_t\!\left[\|\pi x_{t+1}\|^2\right]
\le
\left(
1-\frac{n}{n-1}\alpha
\right)^2
\|\pi x_t+\pi y_t\|^2.
\]
Using Young's inequality, this gives
\[
\mathbb E_t\!\left[\|\pi x_{t+1}\|^2\right]
-
\|\pi x_t\|^2
\le
-\frac{n}{n-1}\alpha \|\pi x_t\|^2
+
\frac{n-1}{n\alpha}\|\pi y_t\|^2.
\]
Equivalently, this is of the form
\[
\mathbb E_t\!\left[\|\pi x_{t+1}\|^2\right]
-
\|\pi x_t\|^2
\le
-\frac{\alpha}{\chi_1}\|\pi x_t\|^2
+
\frac{\chi_1}{\alpha}\|\pi y_t\|^2,
\qquad
\chi_1=\frac{n-1}{n}.
\]

For the $2$-Peer scheme, using
\[
    \mathbb E[\Lambda_t]
    =
    \frac{n}{2(n-1)}\pi,
    \qquad
    \mathbb E[\Lambda_t^2]
    =
    \frac{3n}{8(n-1)}\pi,
\]
we obtain
\[
\mathbb E_t\!\left[\|\pi x_{t+1}\|^2\right]
\le
\left(
1-2\frac{n}{n-1}\alpha
+
\frac{3n}{4(n-1)}\alpha^2
\right)
\|\pi x_t+\pi y_t\|^2.
\]
Let
\[
    q_2
    =
    1-2\frac{n}{n-1}\alpha
    +
    \frac{3n}{4(n-1)}\alpha^2.
\]
Then, by Young's inequality,
\[
q_2\|\pi x_t+\pi y_t\|^2
\le
\sqrt{q_2}\|\pi x_t\|^2
+
\frac{q_2}{1-\sqrt{q_2}}\|\pi y_t\|^2.
\]
Therefore,
\[
\mathbb E_t\!\left[\|\pi x_{t+1}\|^2\right]
-
\|\pi x_t\|^2
\le
-(1-\sqrt{q_2})\|\pi x_t\|^2
+
\frac{q_2}{1-\sqrt{q_2}}\|\pi y_t\|^2.
\]
Equivalently,
\[
\mathbb E_t\!\left[\|\pi x_{t+1}\|^2\right]
-
\|\pi x_t\|^2
\le
-\frac{\alpha}{\chi_2}\|\pi x_t\|^2
+
\frac{\chi_2}{\alpha}\|\pi y_t\|^2,
\]
with
\[
    \chi_2
    =
    \frac{\alpha}{
    1-
    \sqrt{
    1-2\frac{n}{n-1}\alpha
    +
    \frac{3n}{4(n-1)}\alpha^2
    }}.
\]

\medskip

\noindent\noindent
\noindent
\noindent
\textbf{(iii) Accelerated gossip.}
Since \(\pi\Lambda=\Lambda\pi=\Lambda\), applying \(\pi\) to the accelerated update gives
\[
\pi x_{t+1}
=
\bigl((1+\gamma)I-\alpha\Lambda\bigr)\pi x_t
-\gamma \pi x_{t-1}
+
(I-\alpha\Lambda)\pi y_t .
\]
Therefore,
\[
X_{t+1}
=
AX_t+
\begin{pmatrix}
(I-\alpha\Lambda)\pi y_t\\
0
\end{pmatrix}.
\]

with \[
A
=
\begin{pmatrix}
(1+\gamma)I-\alpha\Lambda & -\gamma I\\
I & 0
\end{pmatrix}.
\]

Let \((u,v)\) be an eigenvector of \(A\) with eigenvalue \(r\). Since \(u=rv\), and since \(\Lambda v=\lambda v\) with
\[
\lambda\in\left[\frac1\chi,1\right],
\]
we get
\[
r^2-(1-\alpha\lambda+\gamma)r+\gamma=0.
\]
Set
\[
q:=1-\sqrt{\frac{\alpha}{\chi}},
\qquad
\gamma=q^2.
\]
Then
\[
0\le 1-\alpha\lambda+\gamma\le 2q,
\]
so the roots have modulus at most \(q\). Thus, in an equivalent norm on the lifted disagreement space,
\[
\|AX_t\|
\le
(1-\sqrt{\frac{\alpha}{\chi}})\|X_t\|.
\]
Consequently,
\[
\|X_{t+1}\|
\le
q\|X_t\|
+
\left\|
\begin{pmatrix}
(I-\alpha\Lambda)\pi y_t\\
0
\end{pmatrix}
\right\|.
\]
Since \(0<\alpha\le 1\) and \(\lambda_r(\Lambda)\in[1/\chi,1]\), we have
\[
\|(I-\alpha\Lambda)\pi y_t\|
\le
\|\pi y_t\|.
\]
Therefore, \[ \|X_{t+1}\| \le \left(1-\sqrt{\frac{\alpha}{\chi}}\right)\|X_t\| + \|\pi y_t\|. \] Using Young inequality, \[ \|X_{t+1}\|^2 \le \left(1-\sqrt{\frac{\alpha}{\chi}}\right)^2(1+\nu)\|X_t\|^2 +(1+\frac 1\nu) \|\pi y_t\|^2. \]

Choose
\[
\nu
=
\frac{1}{1-\sqrt{\alpha/\chi}}-1.
\]
Then
\[
\left(1-\sqrt{\frac{\alpha}{\chi}}\right)^2(1+\nu)
=
1-\sqrt{\frac{\alpha}{\chi}},
\]
and
\[
1+\frac1\nu
=
\frac{1}{\sqrt{\alpha/\chi}}
=
\sqrt{\frac{\chi}{\alpha}}.
\]
Therefore,
\[
\|X_{t+1}\|^2
\le
\left(1-\sqrt{\frac{\alpha}{\chi}}\right)\|X_t\|^2
+
\sqrt{\frac{\chi}{\alpha}}\|\pi y_t\|^2.
\]
Subtracting \(\|X_t\|^2\), we obtain
\[
\|X_{t+1}\|^2-\|X_t\|^2
\le
-\sqrt{\frac{\alpha}{\chi}}\|X_t\|^2
+
\sqrt{\frac{\chi}{\alpha}}\|\pi y_t\|^2.
\]
Since \(0<\alpha\le 1\), this also implies the looser bound
\[
\|X_{t+1}\|^2-\|X_t\|^2
\le
-\frac{\alpha}{\sqrt{\chi}}\|X_t\|^2
+
\frac{\sqrt{\chi}}{\alpha}\|\pi y_t\|^2.
\]

\end{proof}

\begin{proposition}
Let $\sigma$ be sampled uniformly at random from the set of permutations of $\{1,\dots,n\}$. Then
\[
    \mathbb{E}_\sigma\bigl[\Lambda^{(1)}_\sigma\bigr]
    =\mathbb{E}_\sigma\bigl[\Lambda^{(2)}_\sigma\bigr]
    =
    \frac{n}{2(n-1)} \pi,
    \qquad
    \mathbb{E}_\sigma\bigl[ \bigl(\Lambda^{(1)}_\sigma\bigr)^2\bigr]=\frac{n}{2(n-1)}\pi  \qquad  \mathbb{E}_\sigma\bigl[ \bigl(\Lambda^{(2)}_\sigma\bigr)^2\bigr]=\frac{3n}{8(n-1)}\pi 
\]
\end{proposition}

\begin{proof}
We first consider the $1$-Peer scheme. For any unordered pair $\{a,b\}$ with $a\neq b$, the probability that $a$ and $b$ are matched together in a uniformly random perfect matching is
\[
    \mathbb{P}\bigl(\{a,b\}\text{ is an edge}\bigr)
    =
    \frac{1}{n-1}.
\]
Therefore,
\[
    \mathbb{E}_\sigma\bigl[\Lambda^{(1)}_\sigma\bigr]
    =
    \frac12
    \sum_{1\leq a<b\leq n}
    \mathbb{P}\bigl(\{a,b\}\text{ is an edge}\bigr)
    (e_a-e_b)(e_a-e_b)^\top .
\]
Using the above probability gives
\[
    \mathbb{E}_\sigma\bigl[\Lambda^{(1)}_\sigma\bigr]
    =
    \frac{1}{2(n-1)}
    \sum_{1\leq a<b\leq n}
    (e_a-e_b)(e_a-e_b)^\top .
\]
The standard identity
\[
    \sum_{1\leq a<b\leq n}
    (e_a-e_b)(e_a-e_b)^\top
    =
    nI - \mathbf{1}\mathbf{1}^\top
    =
    n\pi
\]
then yields
\[
    \mathbb{E}_\sigma\bigl[\Lambda^{(1)}_\sigma\bigr]
    =
    \frac{n}{2(n-1)}\pi.
\]

The proof for the $2$-Peer scheme is analogous. In a uniformly random cycle induced by $\sigma$, any unordered pair $\{a,b\}$ appears as an edge with probability
\[
    \mathbb{P}\bigl(\{a,b\}\text{ is an edge}\bigr)
    =
    \frac{2}{n-1}.
\]
Indeed, once the position of $a$ in the cycle is fixed, the node $b$ can occupy either of the two neighboring positions among the remaining $n-1$ positions. Hence,
\[
    \mathbb{E}_\sigma\bigl[\Lambda^{(2)}_\sigma\bigr]
    =
    \frac14
    \sum_{1\leq a<b\leq n}
    \frac{2}{n-1}
    (e_a-e_b)(e_a-e_b)^\top .
\]
Using the same identity as above, we obtain
\[
    \mathbb{E}_\sigma\bigl[\Lambda^{(2)}_\sigma\bigr]
    =\frac{n}{2(n-1)}\pi.
\]
It remains to compute the second moments. We first consider the $1$-Peer scheme. Recall that
\[
    \Lambda^{(1)}_\sigma
    =
    \frac12
    \sum_{i=1}^{m}
    u_i u_i^\top,
    \qquad
    u_i := e_{\sigma(2i-1)} - e_{\sigma(2i)} .
\]
Since the pairs form a matching, the vectors $u_i$ have disjoint supports. Hence,
\[
    u_i^\top u_j = 0 \quad \text{for } i\neq j,
    \qquad
    u_i^\top u_i = 2.
\]
Therefore,
\[
\begin{aligned}
    \bigl(\Lambda^{(1)}_\sigma\bigr)^2
    &=
    \frac14
    \sum_{i,j=1}^{m}
    u_i u_i^\top u_j u_j^\top  \\
    &=
    \frac14
    \sum_{i,j=1}^{m}
    u_i (u_i^\top u_j) u_j^\top \\
    &=
    \frac14
    \sum_{i=1}^{m}
    2 u_i u_i^\top \\
    &=
    \frac12
    \sum_{i=1}^{m}
    u_i u_i^\top \\
    &=
    \Lambda^{(1)}_\sigma .
\end{aligned}
\]
Taking expectations gives
\[
    \mathbb{E}_\sigma\bigl[(\Lambda^{(1)}_\sigma)^2\bigr]
    =
    \mathbb{E}_\sigma\bigl[\Lambda^{(1)}_\sigma\bigr]
    =
    \frac{n}{2(n-1)}\pi .
\]

We now consider the $2$-Peer scheme. Let
\[
    v_i := e_{\sigma(i)} - e_{\sigma(i+1)},
    \qquad
    \sigma(n+1)=\sigma(1).
\]
Then
\[
    \Lambda^{(2)}_\sigma
    =
    \frac14
    \sum_{i=1}^{n}
    v_i v_i^\top .
\]
For a fixed cycle, $\Lambda^{(2)}_\sigma$ is one fourth of the usual Laplacian
of a cycle graph. Its eigenvalues are therefore
\[
    \lambda_k
    =
    1-\cos\left(\frac{2\pi k}{n}\right),
    \qquad
    k=0,\dots,n-1.
\]
Thus,
\[
\begin{aligned}
    \operatorname{Tr}\left[\bigl(\Lambda^{(2)}_\sigma\bigr)^2\right]
    &=\frac 14
    \sum_{k=0}^{n-1}
    \left(
        1-\cos\left(\frac{2\pi k}{n}\right)
    \right)^2 \\
    &=\frac 14
    \sum_{k=0}^{n-1}
    \left[
        1
        -2\cos\left(\frac{2\pi k}{n}\right)
        +
        \cos^2\left(\frac{2\pi k}{n}\right)
    \right].
\end{aligned}
\]
Using
\[
    \sum_{k=0}^{n-1}
    \cos\left(\frac{2\pi k}{n}\right)
    =
    0,
    \qquad
    \sum_{k=0}^{n-1}
    \cos^2\left(\frac{2\pi k}{n}\right)
    =
    \frac n2,
\]
we obtain
\[
    \operatorname{Tr}\left[\bigl(\Lambda^{(2)}_\sigma\bigr)^2\right]
    =\frac 14(
    n+\frac n2)
    =
    \frac{3n}{8}.
\]

Moreover, the distribution of the random cycle is invariant under relabeling
of the nodes. Hence
\[
    \mathbb{E}_\sigma\bigl[(\Lambda^{(2)}_\sigma)^2\bigr]
\]
must act identically on all directions orthogonal to $\mathbf{1}$ and must
vanish on $\operatorname{span}\{\mathbf{1}\}$. Therefore, there exists a scalar
$\alpha$ such that
\[
    \mathbb{E}_\sigma\bigl[(\Lambda^{(2)}_\sigma)^2\bigr]
    =
    \alpha \pi .
\]
Taking traces yields
\[
    \frac{3n}{8}
    =
    \operatorname{Tr}\left(
        \mathbb{E}_\sigma\bigl[(\Lambda^{(2)}_\sigma)^2\bigr]
    \right)
    =
    \operatorname{Tr}(\alpha \Pi)
    =
    \alpha(n-1).
\]
Hence
\[
    \alpha
    =
    \frac{3n}{8(n-1)}.
\]
Consequently,
\[
    \mathbb{E}_\sigma\bigl[(\Lambda^{(2)}_\sigma)^2\bigr]
    =
    \frac{3n}{8(n-1)}\pi .
\]
This concludes the proof.
\end{proof}
\begin{lemma}[Gradient mismatch with stochastic gradients]\label{lemma:gradient-mismatch}
Assume \(f\) is \(L\)-smooth. Let \(\nabla F(x;\xi)\) be an unbiased stochastic estimator of
\(\nabla f(x)\), satisfying
\[
\mathbb{E}\bigl[\nabla F(x;\xi)\mid x\bigr]=\nabla f(x),
\qquad
\mathbb{E}\bigl[\|\nabla F(x;\xi)-\nabla f(x)\|^2\mid x\bigr]\le \sigma^2 .
\]
Let \(0<\beta\le 1/L\), and define
\[
g_t^i:=x_t^i-x_{t,H}^i,
\]
where
\[
x_{t,k+1}^i
=
x_{t,k}^i-\frac{\beta}{H}\nabla F(x_{t,k}^i;\xi_{t,k}^i),
\qquad
x_{t,0}^i=x_t^i .
\]
Then,
\[
\mathbb E\!\left[
\left\langle g_t^i,\nabla f(\bar x_t)\right\rangle
\right].
\le
6\beta^2L^2 2^H
\|x_t^i-\bar x_t\|^2
+
24\beta^4L^2 2^H
\|\nabla f(\bar x_t)\|^2
+
\left(
12\beta^4L^2 2^H
\right)
\sigma^2
-\frac12
\mathbb E\!\left[
\left\|g^i_t
\right\|^2
\right] -\frac12
\mathbb E\!\left[
\|\nabla f(\bar x_t)\|^2
\right]\\
\]
\end{lemma}
\begin{proof}
We want to upper-bound
\[
\mathbb E\!\left[
\left\langle g_t^i,\nabla f(\bar x_t)\right\rangle
\right].
\]

By the local stochastic update rule,
\[
x_{t,k+1}^i
=
x_{t,k}^i
-
\frac{\beta}{H}
\nabla F(x_{t,k}^i;\xi_{t,k}^i),
\qquad k=0,\dots,H-1.
\]
Iterating this recursion gives
\[
x_{t,H}^i
=
x_t^i
-
\frac{\beta}{H}
\sum_{k=0}^{H-1}
\nabla F(x_{t,k}^i;\xi_{t,k}^i).
\]
Hence
\[
g_t^i
:=
-x_t^i+x_{t,H}^i
=
-\frac{\beta}{H}
\sum_{k=0}^{H-1}
\nabla F(x_{t,k}^i;\xi_{t,k}^i).
\]

Assume that \(\mathcal F_{t,k}^i\) contains all randomness before sampling
\(\xi_{t,k}^i\), so that \(x_{t,k}^i\) and \(\bar x_t\) are
\(\mathcal F_{t,k}^i\)-measurable. Also assume conditional unbiasedness:
\[
\mathbb E\!\left[
\nabla F(x_{t,k}^i;\xi_{t,k}^i)
\mid \mathcal F_{t,k}^i
\right]
=
\nabla f_i(x_{t,k}^i).
\]
Then
\[
\begin{aligned}
\mathbb E\!\left[
\left\langle g_t^i,\nabla f(\bar x_t)\right\rangle
\right]
&=
-\frac{\beta}{H}
\sum_{k=0}^{H-1}
\mathbb E\!\left[
\left\langle
\nabla F(x_{t,k}^i;\xi_{t,k}^i),
\nabla f(\bar x_t)
\right\rangle
\right] \\
&=
-\frac{\beta}{H}
\sum_{k=0}^{H-1}
\mathbb E\!\left[
\left\langle
\nabla f_i(x_{t,k}^i),
\nabla f(\bar x_t)
\right\rangle
\right].
\end{aligned}
\]
In the homogeneous case \(f_i=f\), this becomes
\[
\mathbb E\!\left[
\left\langle g_t^i,\nabla f(\bar x_t)\right\rangle
\right]
=
-\frac{\beta}{H}
\sum_{k=0}^{H-1}
\mathbb E\!\left[
\left\langle
\nabla f(x_{t,k}^i),
\nabla f(\bar x_t)
\right\rangle
\right].
\]
Now, we obtain
\[
\begin{aligned}
\mathbb E\!\left[
\left\langle g_t^i,\nabla f(\bar x_t)\right\rangle
\right]
&=\beta
\mathbb E\!\left[
\left\langle -\frac{1}{H}
\sum_{k=0}^{H-1}
\nabla f(x_{t,k}^i),
\nabla f(\bar x_t)
\right\rangle
\right]\\
&=
-\frac{\beta}{2}\mathbb E\|\frac 1H\sum_{k=0}^{H-1}\nabla f(x_{t,k}^i)\|^2
-\frac{\beta}{2}\mathbb E\|\nabla f(\bar x_t)\|^2 \\
&\quad+
\frac{\beta}{2}
\mathbb E\!\left[
\|\frac 1H\sum_{k=0}^{H-1}(\nabla f(x_{t,k}^i)-\nabla f(\bar x_t))\|^2
\right].
\end{aligned}
\]

Using the inner-loop drift estimate from Lemma \ref{lemma:inner-loop},
\[
\begin{aligned}
\mathbb{E}\|x_{t,k}^i-\bar x_t\|^2
&\leq 
3 \left(1+\frac{\beta L}{H}\right)^{2k}
\mathbb{E}\|x_t^i-\bar x_t\|^2 \\
&\quad+
3\frac{\mathbb{E}\|\nabla f(\bar x_t)\|^2}{L^2}
\left(\left(1+\frac{\beta L}{H}\right)^k-1\right)^2 \\
&\quad+
3\frac{\beta\sigma^2}{L}
\frac{
\left(1+\frac{\beta L}{H}\right)^{2k}-1
}{
2+\frac{\beta L}{H}
},
\end{aligned}
\]
and summing over \(k=0,\dots,H-1\), one obtains
\[
\begin{aligned}
\mathbb E\!\left[
\sum_{k=0}^{H-1}\|x_{t,k}^i-\bar x_t\|^2
\right]
&\leq
3\mathbb E\|x_t^i-\bar x_t\|^2
\sum_{k=0}^{H-1}
\left(1+\frac{\beta L}{H}\right)^{2k} \\
&\quad+
3\frac{\mathbb E\|\nabla f(\bar x_t)\|^2}{L^2}
\sum_{k=0}^{H-1}
\left[
\left(1+\frac{\beta L}{H}\right)^k-1
\right]^2 \\
&\quad+
3\frac{\beta\sigma^2}{L}
\frac{1}{2+\frac{\beta L}{H}}
\sum_{k=0}^{H-1}
\left[
\left(1+\frac{\beta L}{H}\right)^{2k}-1
\right].
\end{aligned}
\]
Let
\[
a:=1+\frac{\beta L}{H}.
\]
Then
\[
\sum_{k=0}^{H-1}a^{2k}
=
\frac{a^{2H}-1}{a^2-1},
\]
\[
\sum_{k=0}^{H-1}(a^k-1)^2
=
\frac{a^{2H}-1}{a^2-1}
-
2\frac{a^H-1}{a-1}
+
H,
\]
and
\[
\sum_{k=0}^{H-1}(a^{2k}-1)
=
\frac{a^{2H}-1}{a^2-1}
-
H.
\]
Therefore,
\[
\begin{aligned}
\mathbb E\!\left[
\sum_{k=0}^{H-1}\|x_{t,k}^i-\bar x_t\|^2
\right]
&\le
3\frac{a^{2H}-1}{a^2-1}
\mathbb E\|x_t^i-\bar x_t\|^2 \\
&\quad+
3\frac{\mathbb E\|\nabla f(\bar x_t)\|^2}{L^2}
\left[
\frac{a^{2H}-1}{a^2-1}
-
2\frac{a^H-1}{a-1}
+
H
\right] \\
&\quad+
3\frac{\beta\sigma^2}{L}
\frac{1}{2+\frac{\beta L}{H}}
\left[
\frac{a^{2H}-1}{a^2-1}
-
H
\right].
\end{aligned}
\]

Next, by Jensen's inequality and \(L\)-smoothness,
\[
\begin{aligned}
\mathbb E\!\left[
\left\|
\frac 1H\sum_{k=0}^{H-1}
\bigl(\nabla f(x_{t,k}^i)-\nabla f(\bar x_t)\bigr)
\right\|^2
\right]
&\le
\frac 1H
\sum_{k=0}^{H-1}
\mathbb E
\left\|
\nabla f(x_{t,k}^i)-\nabla f(\bar x_t)
\right\|^2 \\
&\le
\frac{L^2}{H}
\sum_{k=0}^{H-1}
\mathbb E
\left\|
x_{t,k}^i-\bar x_t
\right\|^2 .
\end{aligned}
\]
Thus,
\[
\begin{aligned}
\mathbb E\!\left[
\left\langle g_t^i,\nabla f(\bar x_t)\right\rangle
\right]
&\le
-\frac{\beta}{2}
\mathbb E\|\nabla f(\bar x_t)\|^2 \\
&\quad+
\frac{\beta L^2}{2H}
\sum_{k=0}^{H-1}
\mathbb E
\left\|
x_{t,k}^i-\bar x_t
\right\|^2 .
\end{aligned}
\]
Substituting the previous drift-sum bound yields
\[
\begin{aligned}
\mathbb E\!\left[
\left\langle g_t^i,\nabla f(\bar x_t)\right\rangle
\right]
&\le -\frac{\beta}{2}\mathbb E\|\frac 1H\sum_{k=0}^{H-1}\nabla f(x_{t,k}^i)\|^2
-\frac{\beta}{2}
\mathbb E\|\nabla f(\bar x_t)\|^2 \\
&\quad+
\frac{3\beta L^2}{2H}
\frac{a^{2H}-1}{a^2-1}
\mathbb E\|x_t^i-\bar x_t\|^2 \\
&\quad+
\frac{3\beta}{2H}
\left[
\frac{a^{2H}-1}{a^2-1}
-
2\frac{a^H-1}{a-1}
+
H
\right]
\mathbb E\|\nabla f(\bar x_t)\|^2 \\
&\quad+
\frac{3\beta^2L\sigma^2}{2H}
\frac{1}{2+\frac{\beta L}{H}}
\left[
\frac{a^{2H}-1}{a^2-1}
-
H
\right].
\end{aligned}
\]

Since
\[
a-1=\frac{\beta L}{H},
\qquad
a^2-1
=
\frac{\beta L}{H}
\left(2+\frac{\beta L}{H}\right),
\]
we have
\[
\frac{1}{H}\frac{a^{2H}-1}{a^2-1}
=
\frac{1}{\beta L}
\frac{a^{2H}-1}{2+\frac{\beta L}{H}}.
\]
Therefore,
\[
\begin{aligned}
\mathbb E\!\left[
\left\langle g_t^i,\nabla f(\bar x_t)\right\rangle
\right]
&\le
-\frac{\beta}{2}
\mathbb E\|\nabla f(\bar x_t)\|^2 -\frac{\beta}{2}\mathbb E\|\frac 1H\sum_{k=0}^{H-1}\nabla f(x_{t,k}^i)\|^2\\
&\quad+
\frac{3L}{2}
\frac{
\left(1+\frac{\beta L}{H}\right)^{2H}-1
}{
2+\frac{\beta L}{H}
}
\mathbb E\|x_t^i-\bar x_t\|^2 \\
&\quad+
\frac{3\beta}{2}
\left[
\frac{1}{\beta L}
\frac{
\left(1+\frac{\beta L}{H}\right)^{2H}-1
}{
2+\frac{\beta L}{H}
}
-
\frac{2}{\beta L}
\left(
\left(1+\frac{\beta L}{H}\right)^H-1
\right)
+
1
\right]
\mathbb E\|\nabla f(\bar x_t)\|^2 \\
&\quad+
\frac{3\beta^2L\sigma^2}{2}
\frac{1}{2+\frac{\beta L}{H}}
\left[
\frac{1}{\beta L}
\frac{
\left(1+\frac{\beta L}{H}\right)^{2H}-1
}{
2+\frac{\beta L}{H}
}
-
1
\right].
\end{aligned}
\]

Now assume \(0<\beta L\le 1\). Then
\[
\left(1+\frac{\beta L}{H}\right)^H\le e^{\beta L},
\qquad
\left(1+\frac{\beta L}{H}\right)^{2H}\le e^{2\beta L},
\]
and
\[
2+\frac{\beta L}{H}\ge 2.
\]
Hence
\[
\frac{
\left(1+\frac{\beta L}{H}\right)^{2H}-1
}{
2+\frac{\beta L}{H}
}
\le
\frac{e^{2\beta L}-1}{2}
\le
\beta L e^{2\beta L}
\le
\beta L e^2.
\]
Thus,
\[
\frac{3L}{2}
\frac{
\left(1+\frac{\beta L}{H}\right)^{2H}-1
}{
2+\frac{\beta L}{H}
}
\mathbb E\|x_t^i-\bar x_t\|^2
\le
\frac{3}{2}\beta L^2 e^2
\mathbb E\|x_t^i-\bar x_t\|^2.
\]

Moreover,
\[
\left(
\left(1+\frac{\beta L}{H}\right)^k-1
\right)^2
\le
\left(e^{\beta L}-1\right)^2
\le
\beta^2L^2 e^2,
\]
so
\[
\frac{3\beta}{2H}
\sum_{k=0}^{H-1}
\left[
\left(1+\frac{\beta L}{H}\right)^k-1
\right]^2
\mathbb E\|\nabla f(\bar x_t)\|^2
\le
\frac{3}{2}\beta^3L^2e^2
\mathbb E\|\nabla f(\bar x_t)\|^2.
\]
Similarly,
\[
\frac{
\left(1+\frac{\beta L}{H}\right)^{2k}-1
}{
2+\frac{\beta L}{H}
}
\le
\frac{e^{2\beta L}-1}{2}
\le
\beta L e^2,
\]
and therefore
\[
\frac{3\beta^2L\sigma^2}{2H}
\sum_{k=0}^{H-1}
\frac{
\left(1+\frac{\beta L}{H}\right)^{2k}-1
}{
2+\frac{\beta L}{H}
}
\le
\frac{3}{2}\beta^3L^2e^2\sigma^2.
\]

Combining these estimates, we obtain
\[
\begin{aligned}
\mathbb E\!\left[
\left\langle g_t^i,\nabla f(\bar x_t)\right\rangle
\right]
&\le
-\frac{\beta}{2}
\mathbb E\|\nabla f(\bar x_t)\|^2
+
\frac{3}{2}\beta^3L^2e^2
\mathbb E\|\nabla f(\bar x_t)\|^2 \\
&\quad+
\frac{3}{2}\beta L^2e^2
\mathbb E\|x_t^i-\bar x_t\|^2
+
\frac{3}{2}\beta^3L^2e^2\sigma^2-\frac{\beta}{2}\mathbb E\|\frac 1H\sum_{k=0}^{H-1}\nabla f(x_{t,k}^i)\|^2.
\end{aligned}
\]
Equivalently,
\[
\boxed{
\begin{aligned}
\mathbb E\!\left[
\left\langle g_t^i,\nabla f(\bar x_t)\right\rangle
\right]
&\le
-\beta\left(
\frac12-\frac32\beta^2L^2e^2
\right)
\mathbb E\|\nabla f(\bar x_t)\|^2 -\frac{\beta}{2}\mathbb E\|\frac 1H\sum_{k=0}^{H-1}\nabla f(x_{t,k}^i)\|^2\\
&\quad+
\frac{3}{2}\beta L^2e^2
\mathbb E\|x_t^i-\bar x_t\|^2
+
\frac{3}{2}\beta^3L^2e^2\sigma^2 .
\end{aligned}
}
\]
This proves the desired bound whenever \(0<\beta L\le 1\).
\end{proof}

\begin{lemma}[Inner-loop drift in norm with stochastic gradients]
\label{lemma:inner-loop}Assume that \(f\) is \(L\)-smooth. Let
\(\nabla F(x;\xi)\) be an unbiased stochastic estimator of
\(\nabla f(x)\), with variance bounded by \(\sigma^2\), i.e.,
\[
\mathbb{E}\bigl[\nabla F(x;\xi)\bigr]=\nabla f(x),
\qquad
\mathbb{E}\bigl[\|\nabla F(x;\xi)-\nabla f(x)\|^2\bigr]\le \sigma^2.
\]
Suppose that the inner-loop iterates satisfy
\[
x_{t,k+1}^i
=
x_{t,k}^i-\frac{\beta}{H}\nabla F(x_{t,k}^i;\xi^i_{t,k}),
\qquad
x_{t,0}^i=x_t^i.
\]
Then, for every \(k\in\{0,\dots,H\}\),

\begin{align}
\mathbb{E}\bigl[\|x_{t,k}^i-\bar x_t\|^2\bigr|x_t]\leq 
3 \left(1+\frac{\beta L}{H}\right)^{2k}
\mathbb{E}\|x_t^i-\bar x_t\|^2
+
3\frac{\mathbb{E}\|\nabla f(\bar x_t)\|^2}{L^2}
\left(\left(1+\frac{\beta L}{H}\right)^k-1\right)^2 
+
3\frac{\beta\sigma^2}{L}
\frac{
\left(1+\frac{\beta L}{H}\right)^{2k}-1
}{
2+\frac{\beta L}{H}
}
\end{align} 

\end{lemma}

\begin{proof}
Let
\[
d_{t,k}^i:=x_{t,k}^i-\bar x_t,
\qquad
\eta_{t,k}^i
:=
\nabla F(x_{t,k}^i;\xi^i_{t,k})
-\nabla f(x_{t,k}^i).
\]
Then
\[
d_{t,k+1}^i
=
d_{t,k}^i
-\frac{\beta}{H}
\bigl(
\nabla f(x_{t,k}^i)+\eta_{t,k}^i
\bigr).
\]
Taking norms and using the triangle inequality gives
\[
\|d_{t,k+1}^i\|
\le
\|d_{t,k}^i\|
+
\frac{\beta}{H}
\|\nabla f(x_{t,k}^i)\|
+
\frac{\beta}{H}
\|\eta_{t,k}^i\|.
\]
Since \(f\) is \(L\)-smooth, its gradient is \(L\)-Lipschitz. Hence
\[
\|\nabla f(x_{t,k}^i)\|
\le
\|\nabla f(\bar x_t)\|
+
\|\nabla f(x_{t,k}^i)-\nabla f(\bar x_t)\|
\le
\|\nabla f(\bar x_t)\|
+
L\|d_{t,k}^i\|.
\]
Therefore,
\[
\|d_{t,k+1}^i\|
\le
\left(1+\frac{\beta L}{H}\right)
\|d_{t,k}^i\|
+
\frac{\beta}{H}\|\nabla f(\bar x_t)\|
+
\frac{\beta}{H}\|\eta_{t,k}^i\|.
\]
Taking conditional expectation with respect to the randomness at step \(k\), we obtain
\[
\|d_{t,k+1}^i\|
\le
\left(1+\frac{\beta L}{H}\right)
\|d_{t,k}^i\|
+
\frac{\beta}{H}\|\nabla f(\bar x_t)\|
+
\frac{\beta}{H}
\|\eta_{t,k}^i\|.
\]
Let $q=1+\frac{\beta L}H$, iterating this recursion from \(0\) to \(k-1\) yields
\[
\|d_{t,k}^i\|
\le
q^k \|d_{t,0}^i\|
+
\frac{\beta}{H}
\sum_{j=0}^{k-1} q^j
\|\nabla f(\bar x_t)\|
+\frac{\beta}{H}
\sum_{j=0}^{k-1} q^j \|\eta_{t,k}^i\|.
\]
Since \(d_{t,0}^i=x_t^i-\bar x_t\), we obtain
\[
\|x_{t,k}^i-\bar x_t\|
\le
\left(1+\frac{\beta L}{H}\right)^k
\|x_t^i-\bar x_t\|
+
\frac{\beta}{H}
\|\nabla f(\bar x_t)\|\
\sum_{j=0}^{k-1}
\left(1+\frac{\beta L}{H}\right)^j+\frac \beta H\sum_{j=0}^{k-1} q^j \|\eta_{t,k}^i\|
\]
Finally, if \(L>0\), then
\[
\sum_{j=0}^{k-1}
\left(1+\frac{\beta L}{H}\right)^j
=
\frac{
\left(1+\frac{\beta L}{H}\right)^k-1
}{
\frac{\beta L}{H}
}.
\]
Therefore,
\[
\|x_{t,k}^i-\bar x_t\|
\le
\left(1+\frac{\beta L}{H}\right)^k
\|x_t^i-\bar x_t\|
+
\frac{\beta}{H}
\|\nabla f(\bar x_t)\|\
\frac{
\left(1+\frac{\beta L}{H}\right)^k-1
}{
\frac{\beta L}{H}
}+\frac \beta H\sum_{j=0}^{k-1} q^j \|\eta_{t,k}^i\|
\]
Consequently,
\[
\begin{aligned}
\|x_{t,k}^i-\bar x_t\|^2
&\le
3 (1+\frac{\beta L}{H})^{2k}
\|x_t^i-\bar x_t\|^2
+
3\frac{\|\nabla f(\bar x_t)\|^2}{L^2}
((1+\frac{\beta L}{H})^k-1)^2  \\
&\qquad
+
3\frac{\beta^2}{H}
\sum_{j=0}^{k-1} (1+\frac{\beta L}{H})^{2j}
\|\eta_{t,j}^i\|^2
\end{aligned}
\]

and

\[
\begin{aligned}
\mathbb{E}\|x_{t,k}^i-\bar x_t\|^2
&\le
3 \left(1+\frac{\beta L}{H}\right)^{2k}
\mathbb{E}\|x_t^i-\bar x_t\|^2
+
3\frac{\mathbb{E}\|\nabla f(\bar x_t)\|^2}{L^2}
\left(\left(1+\frac{\beta L}{H}\right)^k-1\right)^2  \\
&\qquad
+
3\frac{\beta^2\sigma^2}{H}
\sum_{j=0}^{k-1}
\left(1+\frac{\beta L}{H}\right)^{2j} \\
&=
3 \left(1+\frac{\beta L}{H}\right)^{2k}
\mathbb{E}\|x_t^i-\bar x_t\|^2
+
3\frac{\mathbb{E}\|\nabla f(\bar x_t)\|^2}{L^2}
\left(\left(1+\frac{\beta L}{H}\right)^k-1\right)^2  \\
&\qquad
+
3\frac{\beta^2\sigma^2}{H}
\frac{
\left(1+\frac{\beta L}{H}\right)^{2k}-1
}{
\left(1+\frac{\beta L}{H}\right)^2-1
} \\
&=
3 \left(1+\frac{\beta L}{H}\right)^{2k}
\mathbb{E}\|x_t^i-\bar x_t\|^2
+
3\frac{\mathbb{E}\|\nabla f(\bar x_t)\|^2}{L^2}
\left(\left(1+\frac{\beta L}{H}\right)^k-1\right)^2  \\
&\qquad
+
3\frac{\beta\sigma^2}{L}
\frac{
\left(1+\frac{\beta L}{H}\right)^{2k}-1
}{
2+\frac{\beta L}{H}
}.
\end{aligned}
\]

This proves the desired bound.
\end{proof}

\FloatBarrier

\section{Experimental Protocol}\label{app:exp:protocol}

\paragraph{Models.}
We pretrain Llama-3-style decoder-only Transformer models~\citep{grattafiori2024llama}. \Cref{tab:model-configs} reports the size-dependent architecture parameters. Both model configurations use causal attention, bf16 training, depth initialization, a context length of 2048 tokens, RoPE with \(\theta=10{,}000\), layer-normalization epsilon \(10^{-5}\), an FFN dimension multiplier of 1 with dimensions rounded to a multiple of 256. Flex attention is disabled in both configurations.

\begin{table}[t]
\caption{Model configurations used in the pre-training experiments. The Q/K/V head dimension is \(d_{\mathrm{model}} / n_{\mathrm{heads}}\).}
\label{tab:model-configs}
\centering
\begin{tabular}{lcccccc}
\toprule
\textbf{Model} & \textbf{Layers} & \(\mathbf{d_{\mathrm{model}}}\) & \textbf{Attention heads} & \textbf{KV heads} & \textbf{Q/K/V head dim.} & \textbf{Token Budget} \\
\midrule
134M & 12 & 768 & 12 & 12 & 64 & 2.7 B \\
551M & 16 & 1536 & 24 & 24 & 64 & 11 B\\
\bottomrule
\end{tabular}
\end{table}

\paragraph{Data.}
We train on FineWeb~\citep{penedo2024fineweb}. The 134M-parameter model is trained on 2.7B tokens, and the 551M-parameter model is trained on 11B tokens, following the approximate 20-token-per-parameter allocation of \citet{hoffmann2022training}. We tokenize with the LLaMA-2 tokenizer~\citep{touvron2023llama2} provided in the Hugging Face repository \texttt{\href{https://huggingface.co/togethercomputer/LLaMA-2-7B-32K/blob/main/tokenizer.model}{togethercomputer/LLaMA-2-7B-32K}}, which has a 32K-token vocabulary, and pack examples into sequences of length 2048. Unless stated otherwise, all experiments use a global batch size of 2M tokens over the \(n\) workers.

\paragraph{Optimization.}
All methods use AdamW~\citep{loshchilov2018decoupled} for local optimization. We fix the weight decay at \(0.1\), set the momentum coefficients to \((\beta_1,\beta_2)=(0.9,0.95)\), and tune the peak local learning rate for each method. The learning-rate schedule uses a 10\% warmup phase that starts at 1\% of the peak learning rate and then follows cosine decay. For DiLoCo-style methods, each worker performs \(H\) local updates before the outer step. In \diloco, workers compute a pseudo-gradient by \allreduce and apply it through an outer optimizer; in \ours, this synchronized step is replaced by decentralized communication over the selected topology. Both \diloco and \ours use a Nesterov momentum optimizer~\citep{sutskever2013on} for the outer step, with learning rate and momentum tuned independently for each method.

\paragraph{Communication and implementation.}
For \ours-1-Peer, we enforce a non-overlapping symmetric pairing at each synchronization step by randomly permuting the workers and pairing consecutive workers, so each worker communicates with exactly one peer per round. For \ours-2-Peer, we use the randomized procedure from \cref{sec:k-peer}: after randomly permuting the workers, each worker communicates with its two adjacent workers in the induced ring. Experiments are conducted on clusters with 4 NVIDIA H100 GPUs per node. Intra-node communication uses NVIDIA NVLink, and inter-node communication uses InfiniBand. All methods are implemented in PyTorch using NCCL as the communication backend and launched with \texttt{torchrun}. For all 134M-parameter experiments, and for 551M-parameter experiments with \(n>8\), we simulate multiple logical workers per GPU.

\paragraph{Hyperparameters.}
We fixed AdamW weight decay and betas as described above and swept the learning rate and the outer optimizer hyperparameters. Taking into account the reduced computational cost, larger hyperparameter searches were conducted on the 134M-parameter model. Local learning rates were searched in the range \(1\times 10^{-3}\)--\(8\times 10^{-3}\) with a step-size of \(1 \times 10^{-3}\), while outer learning rate and outer momentum were searched in \(\{0.,4, 0.6, 0.8, 1.0\}\) and \(\{0.5, 0.7, 0.9\}\), respectively. Then, a finer search with a step size of 0.1 was carried out around the optimal values found in the first sweep. For the 551M-parameter model, we tuned the local learning rate and weight decay for \ddp and \diloco, while for \ours, DAdam, and \versionD the learning rate was tuned around the optimal value of baselines keeping the weight decay fixed. Meanwhile, the outer optimizer was minimally tuned around the optimal values for the smaller model.

\subsection{1-Peer and 2-Peer Topologies}

At each communication round \(t\), we construct a sparse undirected communication graph \(G_t = (V,E_t)\) where \(V=\{1,\ldots,n\}\) is the set of workers and \(E_t\) is the set of active peer-to-peer exchanges for that round.

Unless otherwise stated, we assume that the underlying admissible graph is complete: every pair of workers can potentially communicate, but only a sparse subset of edges is activated at each round. To sample this subset, we first draw a random permutation of the workers and interpret it as a cycle. Each worker then selects the peer on its right or the 2 closest peers on this random cycle. Finally, we symmetrize the selections: an undirected edge \(\{i,j\}\) is included in \(E_t\) whenever either worker \(i\) selects worker \(j\), or worker \(j\) selects worker \(i\).

A fresh permutation is sampled at every communication round, so the active graph changes over time. This allows workers to communicate with different peers across rounds while keeping the number of active exchanges low. The same construction can be applied to arbitrary admissible topologies by restricting the possible peer selections to the edges of the underlying graph. \Cref{fig:1-Peer,fig:2-Peer} show examples of 1- and 2-Peer topologies that vary from one round to the next.

\begin{figure}
    \centering
    \includegraphics[width=0.8\linewidth]{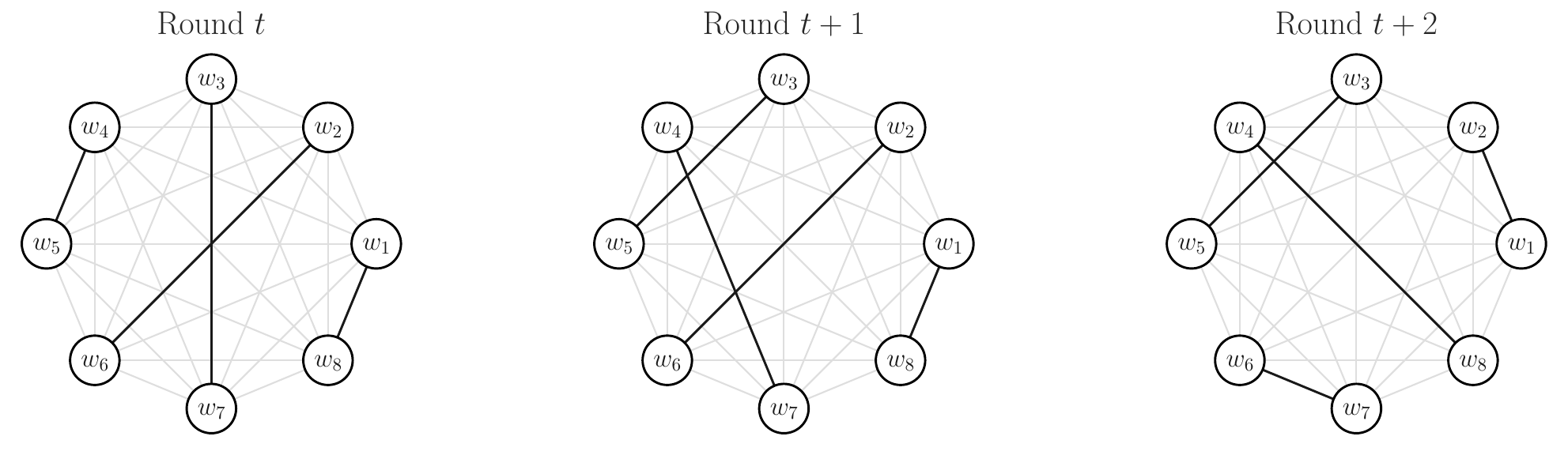}
    \caption{Time-varying 1-Peer graph}
    \label{fig:1-Peer}
\end{figure}

\begin{figure}
    \centering
    \includegraphics[width=0.8\linewidth]{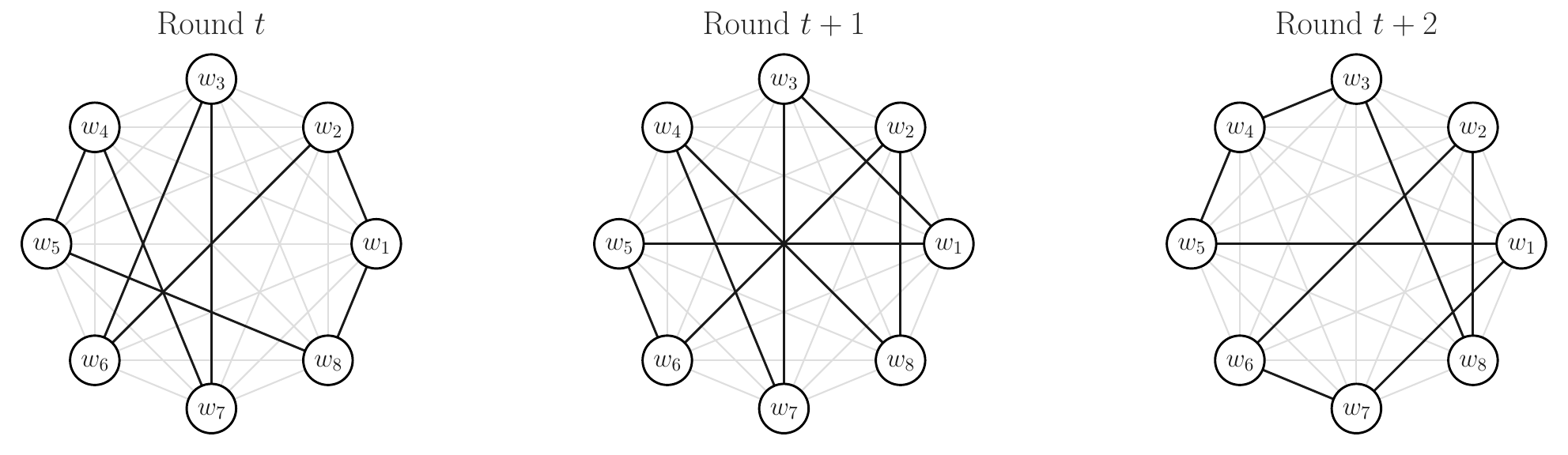}
    \caption{Time-varying 2-Peer graph}
    \label{fig:2-Peer}
\end{figure}

\section{\versionD: Adapting \DAdam to the Multi-Local-Step Setting}\label{app:versionD}

\DAdam is naturally defined in the regime where communication is performed at every optimizer step. Given a row stochastic mixing matrix \(W\) for a specific communication graph, one \DAdam~\citep{wang2024promise} step can be written as
\[
    y_i^{(t)} \gets \mathrm{Optimizer^{in}} \left(x_i^{(t)}, \nabla f_i(x_i^{(t)})\right),
    \qquad
    x_i^{(t+1)}
    \gets
    \sum_{j \in \mathcal N_i} w_{ij} x_j^{(t)}
    +
    \left(y_i^{(t)} - x_i^{(t)}\right).
\]
\DAdam mixes the pre-update parameters \(x_i^{(t)}\) and then applies the local optimizer displacement. Thus, DAdam is directly comparable to our protocol when \(H=1\), where every local optimizer step is followed by communication. For \(H>1\), however, simply delaying the mixing operation until after a block of \(H\) local updates does not recover the original DAdam algorithm: it removes the per-step gossip operation on which DAdam is based and changes the optimization dynamics.

To obtain a controlled baseline with the same communication schedule as \diloco~\citep{douillard2023diloco} and \ours, we therefore define a delayed variant. Starting from \(x_i^{(t)}\), worker \(i\) first performs \(H\) steps with \(\mathrm{Optimizer^{in}}\) to obtain \(y_i^{(t, H)}\). It then forms a \DAdam-style disagreement term using the parameters before local optimization and applies the outer update
\[
    x_i^{(t+1)} \gets \mathrm{Optimizer^{out}}\left(y_i^{(t, H)},\ \sum_{i=1}^n w_{ij} \left(x_i^{(t)} - x_j^{(t)}\right)\right)
\]
When \(H=1\) and \(\mathrm{Optimizer^{out}}\) is SGD with fixed unit learning rate, this reduces to the usual \DAdam update. For \(H>1\), the disagreement correction is stale with respect to the current local trajectories, because it depends only on \(x_j^{(t)}\) and not on the post-local states. Therefore, we interpret this method as a \DAdam adaptation to the multi-step setting, rather than as the original overlapped \DAdam algorithm.

\section{Additional Results}
We provide additional single-step results to complement the comparison in \Cref{sec:1-step-eval}. \Cref{tab:app:small-1step} extends \Cref{tab:small-1step} for the smaller model, while \Cref{tab:app:medium-1step} extends \Cref{tab:medium-1step-scaling} for the 551M-parameter model.

Across these configurations, the qualitative behavior remains consistent with the main results. \ours with randomized communications gives the strongest decentralized results, with the 2-Peer variant performing best overall. In particular, in the 8-replica setting, \ours-2-Peer can even match the reference performance for the larger model. The 1-Peer variant, instead, remains competitive, while benefiting from an even saparser communication. In contrast, \DAdam does not show the same benefits from randomized communication.

\begin{table}[t] \centering \caption{\textbf{One-step Decentralized Methods.} Results for the 134M-parameter model, using \textbf{one local step} per communication round. We compare \ddp to decentralized methods \ours and \DAdam across 8, 16, and 32 replicas.} \label{tab:app:small-1step} {\fontsize{9}{10}\selectfont \begin{tabular}{@{}lcccccc@{}} \toprule & \multicolumn{2}{c}{\textbf{8 replicas}} & \multicolumn{2}{c}{\textbf{16 replicas}} & \multicolumn{2}{c}{\textbf{32 replicas}} \\ \cmidrule(lr){2-3} \cmidrule(lr){4-5} \cmidrule(lr){6-7} AdamW \ddp \textit{(ref.)} & \multicolumn{2}{c}{3.18} & \multicolumn{2}{c}{3.18} & \multicolumn{2}{c}{3.18} \\ \midrule \textbf{Topology} & \textbf{\ours} & \textbf{\DAdam} & \textbf{\ours} & \textbf{\DAdam} & \textbf{\ours} & \textbf{\DAdam} \\ \midrule
2-Peer & 3.19 & 3.25 & 3.22 & 3.28 & 3.28 & 3.34 \\ 
complete & 3.24 & 3.23 & 3.29 & 3.25 & 3.36 & 3.30 \\
1-Peer & 3.20 & 3.25 & 3.24 & 3.29 & 3.31 & 3.37 \\
ring & 3.40 & 3.31 & 3.36 & 3.33 & 3.44 & 3.41 \\
\bottomrule \end{tabular}%
} \end{table}

\begin{table}[]
    \centering
    \caption{\textbf{One-step Decentralized Methods.} Validation loss for the 551M-parameter LLM using 8 and 16 replicas.}
    \label{tab:app:medium-1step}
    \begin{tabular}{llcc}
    \toprule
    \multirow{2}{*}{\textbf{Method}} &
    \multirow{2}{*}{\textbf{Topology}} &
    \multicolumn{2}{c}{\textbf{Val Loss}}\\
    \cmidrule(lr){3-4}
    & &
    \textbf{8 replicas} & \textbf{16 replicas} \\
    \midrule
    AdamW DDP & --
    & 2.64 %
    & 2.64 \\ %
    
    \cdashline{1-4}
    
    DAdam & complete
    & 2.70 %
    & 2.70 \\
    
    DAdam & ring
    & 2.71 %
    & 2.73 \\
    
    \ours & ring
    & 2.66 %
    & 2.70 \\
    
    \ours-1-Peer & complete
    & 2.67 %
    & 2.71 \\
    
    \ours-2-Peer & complete
    & 2.64 %
    & 2.69 \\
    \bottomrule
    \end{tabular}
\end{table}

\section{The Role of the Outer Optimizer}

The outer optimizer is a central component in \ours. Its role, in the decentralized setting, goes beyond local update correction: it helps reducing the communication complexity by acting as a form of communication acceleration.
This makes the choice of the outer optimization method and its hyperparameters particularly important to study. In principle, replacing the communication over a complete graph, which reduces \ours to \diloco, with a randomized peer exchange could make the method more sensitive to the outer optimizer dynamics.

\subsection{The Choice of the Outer Optimization}
\begin{wrapfigure}[16]{r}{0.46\textwidth}
    \centering
    \vspace{-30pt}
    \includegraphics[width=\linewidth]{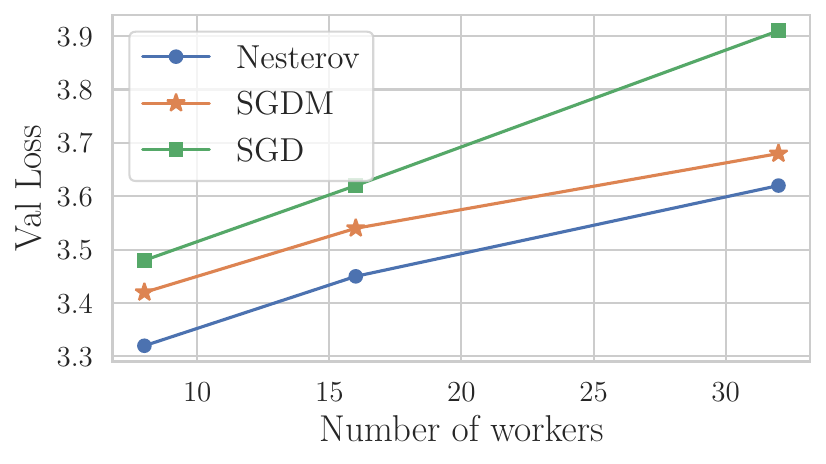}
    \caption{\textbf{Effect of outer optimization method.} Final validation loss of \ours-2-Peer using different outer optimizers as the number of workers increases. Momentum-based methods consistently outperform SGD, with Nesterov momentum achieving the best overall performance.}
    \label{fig:out-opt}
\end{wrapfigure}
We compare different outer optimizers in \ours-2-Peer, including vanilla SGD, SGDM and Nesterov. The motivation behind this experiment lies in the interpretation of the outer optimizer as a communication-acceleration mechanism.

\Cref{fig:out-opt} confirms that momentum-based methods provide an advantage over plain SGD, with Nesterov momentum achieving the best validation loss. We observed the same behavior with up to 32 workers. This confirms the communication-acceleration offered by momentum in the decentralized setting. Although this trend is consistent with the observations previously reported by \citet{douillard2023diloco}, it is important to verify that these conclusions are also valid in the decentralized sparse-communication regime considered by \ours.

\subsection{Hyperparameters Sensitivity}
\begin{figure}[t]
    \centering
    \begin{subfigure}[b]{0.48\textwidth}
        \centering
        \includegraphics[width=\textwidth]{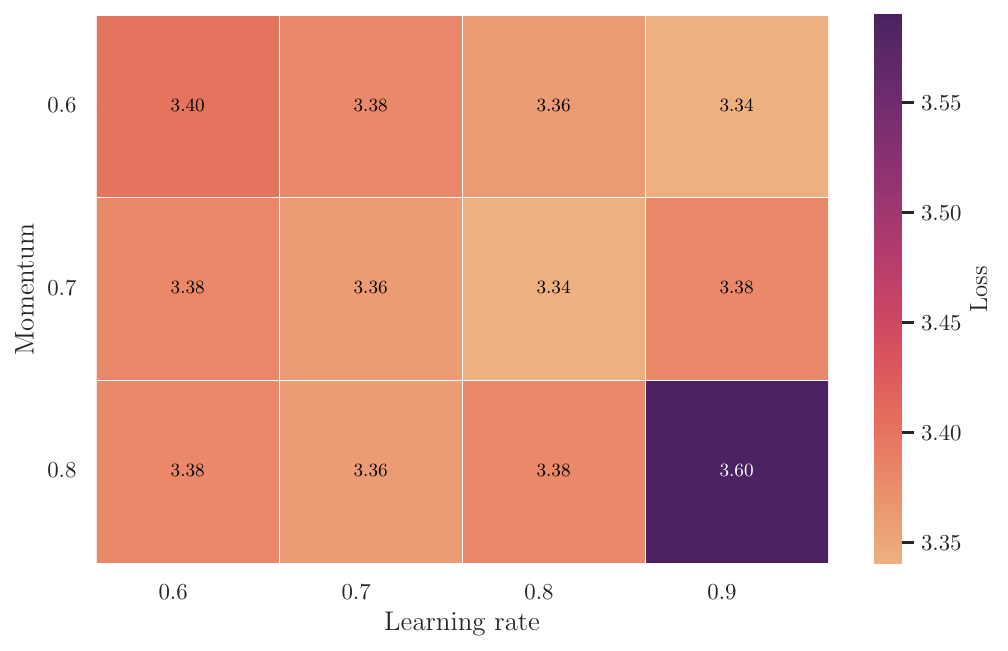}
        \caption{\ours-2-Peer}
        \label{fig:ours-hp-sens}
    \end{subfigure}
    \hfill
    \begin{subfigure}[b]{0.48\textwidth}
        \centering
        \includegraphics[width=\textwidth]{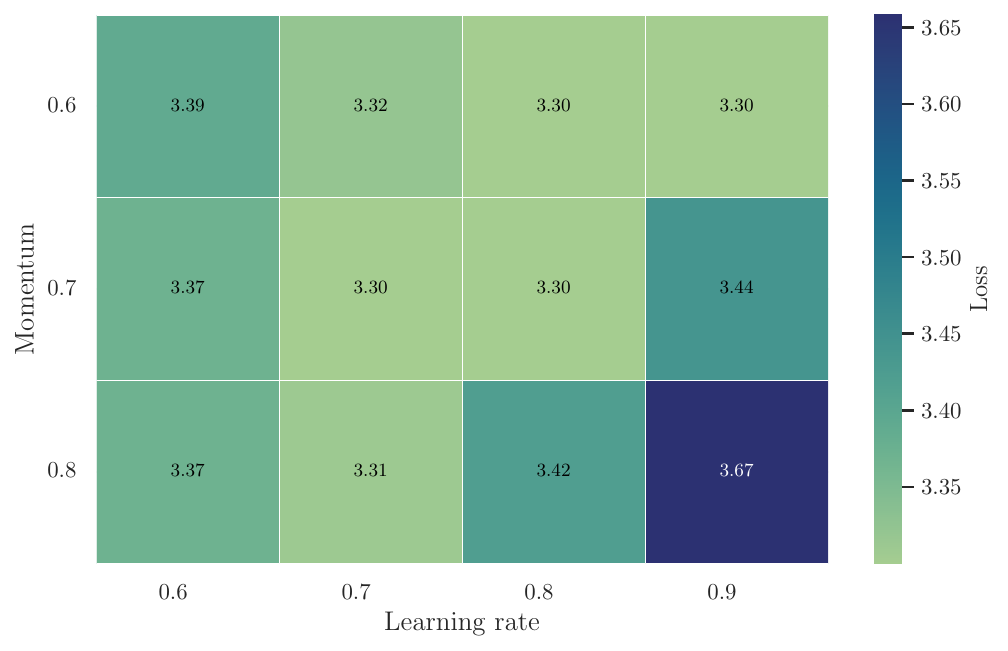}
        \caption{\ours complete graph}
        \label{fig:diloco-hp-sens}
    \end{subfigure}
    \caption{\textbf{Outer Optimizer Hyperparameters Sensitivity.} Validation loss under different learning rates and momentum of the outer optimizer in the 8-worker setting for the 134M-parameter model. Comparing \ours-2-Peer to \ours communcating on the complete graph, both methods remain stable in similar regions of the sweep.}
    \label{fig:hp-sens}
\end{figure}

Communicating over a decentralized and sparse network, \ours-1-Peer and 2-Peer might be more sensitive to hyperparameters than a communication over a complete graph. \Cref{fig:ours-hp-sens} shows that the potential sensitivity does not appear to be a limitation of the method. \ours-2-Peer remains stable over a broad region of the outer-optimizer hyperparameter space that was swept in the experiments and its best-performing region is not limited to a single configuration. Using \ours trained over a complete graph as a reference (Fig. \ref{fig:diloco-hp-sens}), we observe a comparable level of robustness, indicating that the introduction of randomized sparse communication does not substantially increase the complexity of tuning the outer optimizer. Although \ours-2-Peer exhibits slightly higher validation losses in the best-performing regions of the sweep, it shows lower spikes near the edges of the hyperparameter grid.

\section{Compute Utilization in Bandwidth-Constrained Settings}
\label{app:bandwidth}
We define the theoretical compute utilization of each method as
\[
\frac{T_{\mathrm{compute}}}{T_{\mathrm{compute}} + T_{\mathrm{comm}}},
\]
where \(T_{\mathrm{compute}}\) is the estimated time spent in pure computation and \(T_{\mathrm{comm}}\) is the estimated communication time. We estimate \(T_{\mathrm{compute}}\) from the model FLOP profile. The simulation assumes 16 workers, each with \(8\) GPUs, a theoretical FP16/BF16 peak throughput of \(4.5 \times 10^{15}\) FLOPs/s per GPU, and a machine FLOP utilization of \(40\%\). We then vary the maximum available bandwidth and introduce one bandwidth straggler whose link is limited to \(20\%\) of that bandwidth.

\Cref{fig:compute-util} reports the resulting theoretical trends for a \(70\)B-parameter Llama-3 model~\citep{grattafiori2024llama}. We compare DDP and \diloco, both implemented with \allreduce communication, with \ours-1-Peer and \ours-2-Peer. For \ours, the communication-limited worker may perform fewer local steps, so that its slower communication link is partly offset by reduced local computation before synchronization. In this model, both sparse variants improve compute utilization relative to the \allreduce baselines in bandwidth-constrained regimes. \ours-1-Peer exceeds \(80\%\) utilization once the non-straggler bandwidth reaches roughly 10 Gbps, while \ours-2-Peer exceeds this level at roughly 20 Gbps.

\begin{figure}
    \centering
    \includegraphics[width=0.6\linewidth]{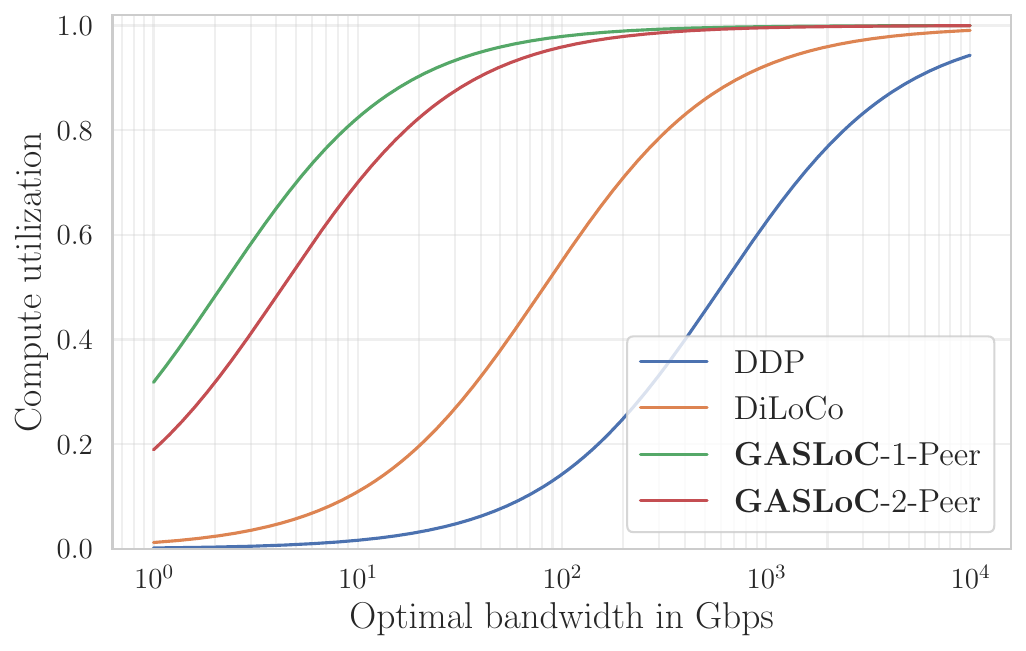}
    \caption{\textbf{Simulated compute utilization.} We report theoretical compute utilization for a \(70\)B-parameter model as the non-straggler bandwidth varies, with one bandwidth straggler limited to \(20\%\) of that bandwidth. DDP and \diloco use \allreduce communication and are therefore bottlenecked by the straggler. \ours-1-Peer and \ours-2-Peer use sparse communication and allow the straggler to perform fewer local steps, which increases estimated compute utilization in communication-constrained regimes.}

    \label{fig:compute-util}
\end{figure}

\FloatBarrier

\FloatBarrier

\newpage
\section*{NeurIPS Paper Checklist}

\if False
The checklist is designed to encourage best practices for responsible machine learning research, addressing issues of reproducibility, transparency, research ethics, and societal impact. Do not remove the checklist: {\bf The papers not including the checklist will be desk rejected.} The checklist should follow the references and follow the (optional) supplemental material.  The checklist does NOT count towards the page
limit. 

Please read the checklist guidelines carefully for information on how to answer these questions. For each question in the checklist:
\begin{itemize}
    \item You should answer \answerYes{}, \answerNo{}, or \answerNA{}.
    \item \answerNA{} means either that the question is Not Applicable for that particular paper or the relevant information is Not Available.
    \item Please provide a short (1--2 sentence) justification right after your answer (even for \answerNA). 
\end{itemize}

{\bf The checklist answers are an integral part of your paper submission.} They are visible to the reviewers, area chairs, senior area chairs, and ethics reviewers. You will also be asked to include it (after eventual revisions) with the final version of your paper, and its final version will be published with the paper.

The reviewers of your paper will be asked to use the checklist as one of the factors in their evaluation. While \answerYes{} is generally preferable to \answerNo{}, it is perfectly acceptable to answer \answerNo{} provided a proper justification is given (e.g., error bars are not reported because it would be too computationally expensive'' or ``we were unable to find the license for the dataset we used''). In general, answering \answerNo{} or \answerNA{} is not grounds for rejection. While the questions are phrased in a binary way, we acknowledge that the true answer is often more nuanced, so please just use your best judgment and write a justification to elaborate. All supporting evidence can appear either in the main paper or the supplemental material, provided in appendix. If you answer \answerYes{} to a question, in the justification please point to the section(s) where related material for the question can be found.

IMPORTANT, please:
\begin{itemize}
    \item {\bf Delete this instruction block, but keep the section heading ``NeurIPS Paper Checklist"},
    \item  {\bf Keep the checklist subsection headings, questions/answers and guidelines below.}
    \item {\bf Do not modify the questions and only use the provided macros for your answers}.
\end{itemize} 
 
\fi

\begin{enumerate}

\item {\bf Claims}
    \item[] Question: Do the main claims made in the abstract and introduction accurately reflect the paper's contributions and scope?
    \item[] Answer: \answerYes{} %
    \item[] Justification: The abstract and introduction state the proposed method, its relationship to gossip-based communication and local steps, the homogeneous-setting convergence result, and the evaluated LLM pre-training regimes. The empirical claims are restricted to the baselines, model sizes, worker counts, and heterogeneous-bandwidth setting studied in \cref{sec:numerical-experiments}.
    \item[] Guidelines:
    \begin{itemize}
        \item The answer \answerNA{} means that the abstract and introduction do not include the claims made in the paper.
        \item The abstract and/or introduction should clearly state the claims made, including the contributions made in the paper and important assumptions and limitations. A \answerNo{} or \answerNA{} answer to this question will not be perceived well by the reviewers. 
        \item The claims made should match theoretical and experimental results, and reflect how much the results can be expected to generalize to other settings. 
        \item It is fine to include aspirational goals as motivation as long as it is clear that these goals are not attained by the paper. 
    \end{itemize}

\item {\bf Limitations}
    \item[] Question: Does the paper discuss the limitations of the work performed by the authors?
    \item[] Answer: \answerYes{} %
    \item[] Justification: The paper discusses limitations including the homogeneous-objective assumptions in the theory, the model scales and worker counts evaluated, the use of a simulated/normalized heterogeneous-bandwidth setting, and the sensitivity of decentralized methods to topology, local-step count, and optimizer tuning.
    \item[] Guidelines:
    \begin{itemize}
        \item The answer \answerNA{} means that the paper has no limitation while the answer \answerNo{} means that the paper has limitations, but those are not discussed in the paper. 
        \item The authors are encouraged to create a separate ``Limitations'' section in their paper.
        \item The paper should point out any strong assumptions and how robust the results are to violations of these assumptions (e.g., independence assumptions, noiseless settings, model well-specification, asymptotic approximations only holding locally). The authors should reflect on how these assumptions might be violated in practice and what the implications would be.
        \item The authors should reflect on the scope of the claims made, e.g., if the approach was only tested on a few datasets or with a few runs. In general, empirical results often depend on implicit assumptions, which should be articulated.
        \item The authors should reflect on the factors that influence the performance of the approach. For example, a facial recognition algorithm may perform poorly when image resolution is low or images are taken in low lighting. Or a speech-to-text system might not be used reliably to provide closed captions for online lectures because it fails to handle technical jargon.
        \item The authors should discuss the computational efficiency of the proposed algorithms and how they scale with dataset size.
        \item If applicable, the authors should discuss possible limitations of their approach to address problems of privacy and fairness.
        \item While the authors might fear that complete honesty about limitations might be used by reviewers as grounds for rejection, a worse outcome might be that reviewers discover limitations that aren't acknowledged in the paper. The authors should use their best judgment and recognize that individual actions in favor of transparency play an important role in developing norms that preserve the integrity of the community. Reviewers will be specifically instructed to not penalize honesty concerning limitations.
    \end{itemize}

\item {\bf Theory assumptions and proofs}
    \item[] Question: For each theoretical result, does the paper provide the full set of assumptions and a complete (and correct) proof?
    \item[] Answer: \answerYes{} %
    \item[] Justification: The main convergence statement explicitly states the smoothness, stochastic-gradient, initialization, step-size, and communication assumptions, and the proof is provided in Appendix A.
    \item[] Guidelines:
    \begin{itemize}
        \item The answer \answerNA{} means that the paper does not include theoretical results. 
        \item All the theorems, formulas, and proofs in the paper should be numbered and cross-referenced.
        \item All assumptions should be clearly stated or referenced in the statement of any theorems.
        \item The proofs can either appear in the main paper or the supplemental material, but if they appear in the supplemental material, the authors are encouraged to provide a short proof sketch to provide intuition. 
        \item Inversely, any informal proof provided in the core of the paper should be complemented by formal proofs provided in appendix or supplemental material.
        \item Theorems and Lemmas that the proof relies upon should be properly referenced. 
    \end{itemize}

    \item {\bf Experimental result reproducibility}
    \item[] Question: Does the paper fully disclose all the information needed to reproduce the main experimental results of the paper to the extent that it affects the main claims and/or conclusions of the paper (regardless of whether the code and data are provided or not)?
    \item[] Answer: \answerYes{} %
    \item[] Justification: \cref{sec:numerical-experiments} and \cref{app:exp:protocol} specify the dataset, preprocessing/tokenization, model family and sizes, token budgets, batch size, optimizers, communication schedules, topologies, hardware, and implementation backend. The algorithmic update is given in \cref{alg:ours-general}, and the \versionD adaptation is described separately in \cref{app:exp:protocol}.
    \item[] Guidelines:
    \begin{itemize}
        \item The answer \answerNA{} means that the paper does not include experiments.
        \item If the paper includes experiments, a \answerNo{} answer to this question will not be perceived well by the reviewers: Making the paper reproducible is important, regardless of whether the code and data are provided or not.
        \item If the contribution is a dataset and\slash or model, the authors should describe the steps taken to make their results reproducible or verifiable. 
        \item Depending on the contribution, reproducibility can be accomplished in various ways. For example, if the contribution is a novel architecture, describing the architecture fully might suffice, or if the contribution is a specific model and empirical evaluation, it may be necessary to either make it possible for others to replicate the model with the same dataset, or provide access to the model. In general. releasing code and data is often one good way to accomplish this, but reproducibility can also be provided via detailed instructions for how to replicate the results, access to a hosted model (e.g., in the case of a large language model), releasing of a model checkpoint, or other means that are appropriate to the research performed.
        \item While NeurIPS does not require releasing code, the conference does require all submissions to provide some reasonable avenue for reproducibility, which may depend on the nature of the contribution. For example
        \begin{enumerate}
            \item If the contribution is primarily a new algorithm, the paper should make it clear how to reproduce that algorithm.
            \item If the contribution is primarily a new model architecture, the paper should describe the architecture clearly and fully.
            \item If the contribution is a new model (e.g., a large language model), then there should either be a way to access this model for reproducing the results or a way to reproduce the model (e.g., with an open-source dataset or instructions for how to construct the dataset).
            \item We recognize that reproducibility may be tricky in some cases, in which case authors are welcome to describe the particular way they provide for reproducibility. In the case of closed-source models, it may be that access to the model is limited in some way (e.g., to registered users), but it should be possible for other researchers to have some path to reproducing or verifying the results.
        \end{enumerate}
    \end{itemize}

\item {\bf Open access to data and code}
    \item[] Question: Does the paper provide open access to the data and code, with sufficient instructions to faithfully reproduce the main experimental results, as described in supplemental material?
    \item[] Answer: \answerYes{} %
    \item[] Justification: The dataset used in the experiments is publicly available, and we will release the code accompanying the paper, including the implementation of \ours, the evaluated baselines, configuration files, and launch instructions. The experimental protocol in \cref{sec:numerical-experiments} and \cref{app:exp:protocol} specifies the dataset, preprocessing, model sizes, optimizer settings, communication schedules, topologies, worker counts, and hardware setup needed to reproduce the main results.
    \item[] Guidelines:
    \begin{itemize}
        \item The answer \answerNA{} means that paper does not include experiments requiring code.
        \item Please see the NeurIPS code and data submission guidelines (\url{https://neurips.cc/public/guides/CodeSubmissionPolicy}) for more details.
        \item While we encourage the release of code and data, we understand that this might not be possible, so \answerNo{} is an acceptable answer. Papers cannot be rejected simply for not including code, unless this is central to the contribution (e.g., for a new open-source benchmark).
        \item The instructions should contain the exact command and environment needed to run to reproduce the results. See the NeurIPS code and data submission guidelines (\url{https://neurips.cc/public/guides/CodeSubmissionPolicy}) for more details.
        \item The authors should provide instructions on data access and preparation, including how to access the raw data, preprocessed data, intermediate data, and generated data, etc.
        \item The authors should provide scripts to reproduce all experimental results for the new proposed method and baselines. If only a subset of experiments are reproducible, they should state which ones are omitted from the script and why.
        \item At submission time, to preserve anonymity, the authors should release anonymized versions (if applicable).
        \item Providing as much information as possible in supplemental material (appended to the paper) is recommended, but including URLs to data and code is permitted.
    \end{itemize}

\item {\bf Experimental setting/details}
    \item[] Question: Does the paper specify all the training and test details (e.g., data splits, hyperparameters, how they were chosen, type of optimizer) necessary to understand the results?
    \item[] Answer: \answerYes{} %
    \item[] Justification: \cref{sec:numerical-experiments} and \cref{app:exp:protocol} describe the FineWeb data, tokenizer, sequence length, model sizes, token budgets, global batch size, local AdamW optimizer, outer optimizer, communication frequency, worker counts, and communication topology.
    \item[] Guidelines:
    \begin{itemize}
        \item The answer \answerNA{} means that the paper does not include experiments.
        \item The experimental setting should be presented in the core of the paper to a level of detail that is necessary to appreciate the results and make sense of them.
        \item The full details can be provided either with the code, in appendix, or as supplemental material.
    \end{itemize}

\item {\bf Experiment statistical significance}
    \item[] Question: Does the paper report error bars suitably and correctly defined or other appropriate information about the statistical significance of the experiments?
    \item[] Answer: \answerNo{} %
    \item[] Justification: Following common practice in LLM pre-training, where full repeated-seed runs are often prohibitively expensive, we report single-run validation loss rather than repeated-seed error bars. To partially mitigate this limitation, we evaluate the same trends across multiple worker counts, communication topologies, local-step regimes, and model scales.
    \item[] Guidelines:
    \begin{itemize}
        \item The answer \answerNA{} means that the paper does not include experiments.
        \item The authors should answer \answerYes{} if the results are accompanied by error bars, confidence intervals, or statistical significance tests, at least for the experiments that support the main claims of the paper.
        \item The factors of variability that the error bars are capturing should be clearly stated (for example, train/test split, initialization, random drawing of some parameter, or overall run with given experimental conditions).
        \item The method for calculating the error bars should be explained (closed form formula, call to a library function, bootstrap, etc.)
        \item The assumptions made should be given (e.g., Normally distributed errors).
        \item It should be clear whether the error bar is the standard deviation or the standard error of the mean.
        \item It is OK to report 1-sigma error bars, but one should state it. The authors should preferably report a 2-sigma error bar than state that they have a 96\% CI, if the hypothesis of Normality of errors is not verified.
        \item For asymmetric distributions, the authors should be careful not to show in tables or figures symmetric error bars that would yield results that are out of range (e.g., negative error rates).
        \item If error bars are reported in tables or plots, the authors should explain in the text how they were calculated and reference the corresponding figures or tables in the text.
    \end{itemize}

\item {\bf Experiments compute resources}
    \item[] Question: For each experiment, does the paper provide sufficient information on the computer resources (type of compute workers, memory, time of execution) needed to reproduce the experiments?
    \item[] Answer: \answerYes{} %
    \item[] Justification: \cref{app:exp:protocol} reports the training protocol like the model sizes, token budgets, worker counts, and logical-worker setup needed to estimate the compute required for each experiment.
    \item[] Guidelines:
    \begin{itemize}
        \item The answer \answerNA{} means that the paper does not include experiments.
        \item The paper should indicate the type of compute workers CPU or GPU, internal cluster, or cloud provider, including relevant memory and storage.
        \item The paper should provide the amount of compute required for each of the individual experimental runs as well as estimate the total compute. 
        \item The paper should disclose whether the full research project required more compute than the experiments reported in the paper (e.g., preliminary or failed experiments that didn't make it into the paper). 
    \end{itemize}
    
\item {\bf Code of ethics}
    \item[] Question: Does the research conducted in the paper conform, in every respect, with the NeurIPS Code of Ethics \url{https://neurips.cc/public/EthicsGuidelines}?
    \item[] Answer: \answerYes{} %
    \item[] Justification: The work is an algorithmic and empirical study of decentralized optimization for LLM pre-training and does not involve human subjects, private user data, deception, or deployment of a model in a real-world decision-making system. The experiments use existing research datasets and standard compute infrastructure.
    \item[] Guidelines:
    \begin{itemize}
        \item The answer \answerNA{} means that the authors have not reviewed the NeurIPS Code of Ethics.
        \item If the authors answer \answerNo, they should explain the special circumstances that require a deviation from the Code of Ethics.
        \item The authors should make sure to preserve anonymity (e.g., if there is a special consideration due to laws or regulations in their jurisdiction).
    \end{itemize}

\item {\bf Broader impacts}
    \item[] Question: Does the paper discuss both potential positive societal impacts and negative societal impacts of the work performed?
    \item[] Answer: \answerYes{} %
    \item[] Justification: The paper discusses the positive impact of reducing communication bottlenecks, which may improve hardware utilization and make distributed pre-training more accessible.
    \item[] Guidelines:
    \begin{itemize}
        \item The answer \answerNA{} means that there is no societal impact of the work performed.
        \item If the authors answer \answerNA{} or \answerNo, they should explain why their work has no societal impact or why the paper does not address societal impact.
        \item Examples of negative societal impacts include potential malicious or unintended uses (e.g., disinformation, generating fake profiles, surveillance), fairness considerations (e.g., deployment of technologies that could make decisions that unfairly impact specific groups), privacy considerations, and security considerations.
        \item The conference expects that many papers will be foundational research and not tied to particular applications, let alone deployments. However, if there is a direct path to any negative applications, the authors should point it out. For example, it is legitimate to point out that an improvement in the quality of generative models could be used to generate Deepfakes for disinformation. On the other hand, it is not needed to point out that a generic algorithm for optimizing neural networks could enable people to train models that generate Deepfakes faster.
        \item The authors should consider possible harms that could arise when the technology is being used as intended and functioning correctly, harms that could arise when the technology is being used as intended but gives incorrect results, and harms following from (intentional or unintentional) misuse of the technology.
        \item If there are negative societal impacts, the authors could also discuss possible mitigation strategies (e.g., gated release of models, providing defenses in addition to attacks, mechanisms for monitoring misuse, mechanisms to monitor how a system learns from feedback over time, improving the efficiency and accessibility of ML).
    \end{itemize}
    
\item {\bf Safeguards}
    \item[] Question: Does the paper describe safeguards that have been put in place for responsible release of data or models that have a high risk for misuse (e.g., pre-trained language models, image generators, or scraped datasets)?
    \item[] Answer: \answerNA{} %
    \item[] Justification: The paper does not release a new pre-trained language model, image generator, or scraped dataset.
    \item[] Guidelines:
    \begin{itemize}
        \item The answer \answerNA{} means that the paper poses no such risks.
        \item Released models that have a high risk for misuse or dual-use should be released with necessary safeguards to allow for controlled use of the model, for example by requiring that users adhere to usage guidelines or restrictions to access the model or implementing safety filters. 
        \item Datasets that have been scraped from the Internet could pose safety risks. The authors should describe how they avoided releasing unsafe images.
        \item We recognize that providing effective safeguards is challenging, and many papers do not require this, but we encourage authors to take this into account and make a best faith effort.
    \end{itemize}

\item {\bf Licenses for existing assets}
    \item[] Question: Are the creators or original owners of assets (e.g., code, data, models), used in the paper, properly credited and are the license and terms of use explicitly mentioned and properly respected?
    \item[] Answer: \answerYes{} %
    \item[] Justification: The paper cites the existing datasets, tokenizer/model references, software frameworks, and baseline methods used in the experiments, and the appendix states the relevant licenses and terms of use for these assets.
    \item[] Guidelines:
    \begin{itemize}
        \item The answer \answerNA{} means that the paper does not use existing assets.
        \item The authors should cite the original paper that produced the code package or dataset.
        \item The authors should state which version of the asset is used and, if possible, include a URL.
        \item The name of the license (e.g., CC-BY 4.0) should be included for each asset.
        \item For scraped data from a particular source (e.g., website), the copyright and terms of service of that source should be provided.
        \item If assets are released, the license, copyright information, and terms of use in the package should be provided. For popular datasets, \url{paperswithcode.com/datasets} has curated licenses for some datasets. Their licensing guide can help determine the license of a dataset.
        \item For existing datasets that are re-packaged, both the original license and the license of the derived asset (if it has changed) should be provided.
        \item If this information is not available online, the authors are encouraged to reach out to the asset's creators.
    \end{itemize}

\item {\bf New assets}
    \item[] Question: Are new assets introduced in the paper well documented and is the documentation provided alongside the assets?
    \item[] Answer: \answerYes{} %
    \item[] Justification: The paper introduces and will release code for \ours as a new research artifact. The released repository will include documentation, configuration files, and instructions for reproducing the main experiments, including the communication topologies, local-step schedules, outer optimizer settings, and baseline implementations used in the paper.
    \item[] Guidelines:
    \begin{itemize}
        \item The answer \answerNA{} means that the paper does not release new assets.
        \item Researchers should communicate the details of the dataset\slash code\slash model as part of their submissions via structured templates. This includes details about training, license, limitations, etc. 
        \item The paper should discuss whether and how consent was obtained from people whose asset is used.
        \item At submission time, remember to anonymize your assets (if applicable). You can either create an anonymized URL or include an anonymized zip file.
    \end{itemize}

\item {\bf Crowdsourcing and research with human subjects}
    \item[] Question: For crowdsourcing experiments and research with human subjects, does the paper include the full text of instructions given to participants and screenshots, if applicable, as well as details about compensation (if any)? 
    \item[] Answer: \answerNA{} %
    \item[] Justification: The paper does not involve crowdsourcing, user studies, annotation tasks, or other research with human subjects.
    \item[] Guidelines:
    \begin{itemize}
        \item The answer \answerNA{} means that the paper does not involve crowdsourcing nor research with human subjects.
        \item Including this information in the supplemental material is fine, but if the main contribution of the paper involves human subjects, then as much detail as possible should be included in the main paper. 
        \item According to the NeurIPS Code of Ethics, workers involved in data collection, curation, or other labor should be paid at least the minimum wage in the country of the data collector. 
    \end{itemize}

\item {\bf Institutional review board (IRB) approvals or equivalent for research with human subjects}
    \item[] Question: Does the paper describe potential risks incurred by study participants, whether such risks were disclosed to the subjects, and whether Institutional Review Board (IRB) approvals (or an equivalent approval/review based on the requirements of your country or institution) were obtained?
    \item[] Answer: \answerNA{} %
    \item[] Justification: The paper does not involve human subjects, crowdsourced participants, or collection of data from study participants, so IRB approval or equivalent human-subjects review is not applicable.
    \item[] Guidelines:
    \begin{itemize}
        \item The answer \answerNA{} means that the paper does not involve crowdsourcing nor research with human subjects.
        \item Depending on the country in which research is conducted, IRB approval (or equivalent) may be required for any human subjects research. If you obtained IRB approval, you should clearly state this in the paper. 
        \item We recognize that the procedures for this may vary significantly between institutions and locations, and we expect authors to adhere to the NeurIPS Code of Ethics and the guidelines for their institution. 
        \item For initial submissions, do not include any information that would break anonymity (if applicable), such as the institution conducting the review.
    \end{itemize}

\item {\bf Declaration of LLM usage}
    \item[] Question: Does the paper describe the usage of LLMs if it is an important, original, or non-standard component of the core methods in this research? Note that if the LLM is used only for writing, editing, or formatting purposes and does \emph{not} impact the core methodology, scientific rigor, or originality of the research, declaration is not required.
    \item[] Answer: \answerNA{} %
    \item[] Justification: LLMs are the model class studied in the experiments, but no external LLM is used as an important, original, or non-standard component of the proposed method.
    \item[] Guidelines:
    \begin{itemize}
        \item The answer \answerNA{} means that the core method development in this research does not involve LLMs as any important, original, or non-standard components.
        \item Please refer to our LLM policy in the NeurIPS handbook for what should or should not be described.
    \end{itemize}

\end{enumerate}

\end{document}